\documentclass[useAMS,usegraphicx,usenatbib]{biom}

\usepackage{bbm}
\usepackage[ruled]{algorithm2e}
\SetAlFnt{\small}
\usepackage{booktabs}

\RequirePackage{epsfig}
\RequirePackage{amssymb}
\RequirePackage{graphicx}
\usepackage{color}
\usepackage{url}

\usepackage{array}
\newcolumntype{H}{>{\setbox0=\hbox\bgroup}c<{\egroup}@{}}

\usepackage{MnSymbol,wasysym}

\DeclareMathOperator{\BAR}{BAR}
\DeclareMathOperator{\BAP}{BAC}

\DeclareMathOperator{\meBAR}{meBAR}

\DeclareMathOperator{\meBSR}{meBBR}
\DeclareMathOperator{\BSR}{BBR}

\usepackage{xr}

\makeatletter
\newcommand*{\addFileDependency}[1]{% argument=file name and extension
	\typeout{(#1)}
	\@addtofilelist{#1}
	\IfFileExists{#1}{}{\typeout{No file #1.}}
}
\makeatother

\usepackage[figuresright]{rotating}
\usepackage{pdflscape}

\title[Automatic algorithm change policies]{Approval policies for modifications to Machine Learning-Based Software as a Medical Device: A study of bio-creep}

\author{Jean Feng$^*$\email{jeanfeng@uw.edu},
Scott Emerson$^{**}$\email{semerson@uw.edu}, and
Noah Simon$^{***}$\email{nrsimon@uw.edu} \\
Department of Biostatistics, University of Washington, Seattle, WA, USA}

\begin{document}

%  This will produce the submission and review information that appears
%  right after the reference section.  Of course, it will be unknown when
%  you submit your paper, so you can either leave this out or put in
%  sample dates (these will have no effect on the fate of your paper in the
%  review process!)

\date{}

%  These options will count the number of pages and provide volume
%  and date information in the upper left hand corner of the top of the
%  first page as in published papers.  The \pagerange command will only
%  work if you place the command \label{firstpage} near the beginning
%  of the document and \label{lastpage} at the end of the document, as we
%  have done in this template.

%  Again, putting a volume number and date is for your own amusement and
%  has no bearing on what actually happens to your paper!

\pagerange{\pageref{firstpage}--\pageref{lastpage}}
\volume{NA}
\pubyear{NA}
\artmonth{NA}
%\volume{64}
%\pubyear{2008}
%\artmonth{December}

%  The \doi command is where the DOI for your paper would be placed should it
%  be published.  Again, if you make one up and stick it here, it means
%  nothing!

\doi{}

%  This label and the label ``lastpage'' are used by the \pagerange
%  command above to give the page range for the article.  You may have
%  to process the document twice to get this to match up with what you
%  expect.  When using the referee option, this will not count the pages
%  with tables and figures.

\label{firstpage}

%  put the summary for your paper here

\begin{abstract}
Successful deployment of machine learning algorithms in healthcare
requires careful assessments of their performance and safety. To date,
the FDA approves locked algorithms prior to marketing and requires
future updates to undergo separate premarket reviews. However, this
negates a key feature of machine learning--the ability to learn from a
growing dataset and improve over time. This paper frames the design of
an approval policy, which we refer to as an automatic algorithmic change
protocol (aACP), as an online hypothesis testing problem. As this process
has obvious analogy with noninferiority testing of new drugs, we
investigate how repeated testing and adoption of modifications might
lead to gradual deterioration in prediction accuracy, also known as ``biocreep'' in the drug development literature. We consider simple policies
that one might consider but do not necessarily offer any error-rate
guarantees, as well as policies that do provide error-rate control. For the
latter, we define two online error-rates appropriate for this context: Bad
Approval Count (BAC) and Bad Approval and Benchmark Ratios (BABR).
We control these rates in the simple setting of a constant population and
data source using policies aACP-BAC and aACP-BABR, which combine
alpha-investing, group-sequential, and gate-keeping methods. In
simulation studies, bio-creep regularly occurred when using policies with
no error-rate guarantees, whereas aACP-BAC and -BABR controlled the
rate of bio-creep without substantially impacting our ability to approve
beneficial modifications.
\end{abstract}

%  Please place your key words in alphabetical order, separated
%  by semicolons, with the first letter of the first word capitalized,
%  and a period at the end of the list.
%


%  As usual, the \maketitle command creates the title and author/affiliations
%  display

\maketitle

%  If you are using the referee option, a new page, numbered page 1, will
%  start after the summary and keywords.  The page numbers thus count the
%  number of pages of your manuscript in the preferred submission style.
%  Remember, ``Normally, regular papers exceeding 25 pages and Reader Reaction
%  papers exceeding 12 pages in (the preferred style) will be returned to
%  the authors without review. The page limit includes acknowledgements,
%  references, and appendices, but not tables and figures. The page count does
%  not include the title page and abstract. A maximum of six (6) tables or
%  figures combined is often required.''

%  You may now place the substance of your manuscript here.  Please use
%  the \section, \subsection, etc commands as described in the user guide.
%  Please use \label and \ref commands to cross-reference sections, equations,
%  tables, figures, etc.
%
%  Please DO NOT attempt to reformat the style of equation numbering!
%  For that matter, please do not attempt to redefine anything!

Due to the rapid development of artificial intelligence (AI) and machine learning (ML), the use of AI/ML-based algorithms has expanded in the medical field.
As such, an increasing number of AI/ML-based Software as a Medical Device (SaMD) are seeking approval from the Center of Diagnostics and Radiologic Health (CDRH) at the US Food and Drug Administration (FDA).
ML algorithms are attractive for their ability to improve over time by training over a growing body of data.
Thus, rather than using a locked algorithm trained on a limited dataset, developers might like to train it further on a much more representative sample of the patient population that can only be obtained after deployment.
To collect input on this regulatory problem, the FDA recently outlined a proposed regulatory framework for modifications to AI/ML-based SaMDs in a discussion paper \citep{Fda2019-kt}.

Regulating evolving algorithms presents new challenges because the CDRH has historically only approved ``locked'' algorithms, i.e. algorithms that do not change after they are approved.
This is a new regulatory problem because updating traditional medical devices and drugs is often logistically difficult whereas updating software is both fast and easy.
%Moreover, ML algorithms have the potential to improve over time by training over a growing body of data.
%So rather than using a locked algorithm that is trained on a limited dataset, developers might like to train it further on a much more representative sample of the patient population that can only be obtained after deployment.
%In the interest of advancing public health, the FDA is interested in developing policies that approve beneficial modifications in a timely manner.

\citet{Fda2019-kt} proposes companies stipulate SaMD Pre-specifications (SPS) and an Algorithm Change Protocol (ACP).
When listing the anticipated modifications in the SPS, it behooves the company to cast as wide a net as possible within FDA-imposed constraints.
The ACP specifies how the company will ensure that their modifications are acceptable for deployment.
Once the FDA approves the SPS and ACP, the company follows these pre-specified procedures to deploy changes without further intervention.
As such, we refer to the ACP in this paper as an ``automatic ACP'' (aACP).
The aACP is the FDA's primary tool for ensuring safety and efficacy of the modifications.
However, specific aACP designs or requirements are noticeably absent from \citet{Fda2019-kt}.
This paper aims to address this gap.

A manufacturer has two potential motivations for changing an AI/ML-based SaMD: to advance public health and to increase their financial wealth.
Modifications that improve performance and usability are encouraged.
On the other hand, changes that do not and are deployed only for the sake of change itself have been used in the past to advance a manufacturer's financial interest and are contrary to public interest.
Historically, such modifications have been used to 1) decrease competition because it is difficult for competitors to compare against an ever-changing benchmark; 2) file for a patent extension and keep prices artificially high; and 3) increase sales for a supposedly new and improved product \citep{Gupta2010-wf, Hitchings2012-km, Gottlieb2019-dl}.
To prevent this type of behavior with drugs and biologics, the FDA regulates modifications through various types of bridging studies (\cite{ICH1998-so}). %ICH E5
Likewise, an aACP should only grant approval to modifications to AI/ML-based SaMD after ensuring safety and efficacy.

This paper provides a framework for designing and evaluating an aACP, considers a variety of aACP designs, and investigates their operating characteristics.
We assume the manufacturer is allowed to propose arbitrary (and possibly deleterious) modifications, which include changes to model parameters, structure, and input features.
For this manuscript, we focus on the setting of a constant population and data source, rather than more complicated settings with significant time trends. % in patient populations, diagnostic methods, etc.
Throughout, we evaluate modifications solely in terms of their operating characteristics.
Thus, the aACPs treat simple models and complex black-box estimators, such as neural networks and boosted gradient trees, the same.
This parallels the drug approval process, which primarily evaluates drugs on their efficacy and safety with respect to some endpoints, even if the biological mechanism is not completely understood.

To our knowledge, there is no prior work that directly addresses the problem of regulating modifications to AI/ML-based SaMD, though many have studied related problems.
In online hypothesis testing, alpha-investing procedures are used to control the online false discovery rate (FDR) \citep{Foster2008-ek, Javanmard2015-az, Ramdas2017-un, Ramdas2018-gh, Zrnic2018-ab}, which is important for companies that test many hypotheses over a long period of time (\citet{Tang2010-wc}).
We will consider aACPs that use alpha-investing to control online error rates; However, we will need to significantly adapt these ideas for use in our context.
In addition, differential privacy methods \citep{Blum2015-hv, Dwork2015-da} have been used to tackle the problem of ranking model submissions to a ML competition, where the submissions are evaluated using the same test data and models are submitted in a sequential and adaptive manner.
Though that problem is related, those approaches cannot evaluate modifications that add previously-unmeasured covariates.
Finally, online learning methods are a major motivation for studying this regulatory problem and can be used to automatically update the model \citep{Shalev-Shwartz2012-vg}.
However, rather than designing bespoke aACPs for online learning methods, we will consider approval policies for arbitrary modifications as a first step.

This paper evaluates the rates at which different policies make bad approvals as well as their rates of approving beneficial modifications.
Due to the analogy between this problem and noninferiority testing of new drugs, we investigate how repeated testing of proposed modifications might lead to gradual deterioration in model performance, also known as ``bio-creep'' \citep{Fleming2008-ok}.
We compare simple aACPs that one might consider, but do not necessarily have error-rate guarantees, to policies that \textit{do} provide error rate control.
For the latter, we define two online error rates appropriate for this context---the expected Bad Approval Count (BAC) and Bad Approval and Benchmark Ratios (BABR)---and control them using policies aACP-BAC and aACP-BABR, respectively.
In simulation studies, bio-creep frequently occurred when using the simple aACPs.
By using aACP-BAC or -BABR instead, we significantly reduce the risk of bio-creep without substantially reducing our power at detecting beneficial modifications.
Based on these findings, we conclude that 1) bio-creep is a major concern when designing an aACP and 2) there are promising solutions for mitigating it without large sacrifices in power.

\section{Motivating examples}
\label{sec:examples}

We present examples of actual AI/ML-based medical devices and discuss possible modifications that manufacturers might consider.
The examples are ordered by increasing regulatory complexity and risk.
Throughout, we only discuss regulating modifications to the software and assume the intended use of the device remains constant.

\subsection{Blood tests using computer vision}
\label{sec:blood}
Sight Diagnostics has developed a device that collects and images blood samples to estimate complete blood count (CBC) parameters.
They are evaluating the device in a clinical trial (ClinicalTrials.gov ID NCT03595501) where the endpoints are the estimated linear regression parameters (slope and intercept) between their CBC parameter estimates and gold standard.

The FDA requires locking the entire procedure, which includes blood collection, imaging, and the ML algorithm, prior to marketing.
Nonetheless, the company might want to improve the accuracy of their test after obtaining regulatory approval.
For instance, they can train more complex models that capture nonlinearities and interactions (between covariates and/or outcomes) or use a different FDA-approved device to image the blood sample.
All these changes have the potential to improve prediction accuracy, though it is not guaranteed.

To regulate such modifications, we will need to define acceptable changes to endpoint values.
This is not straightforward when multiple endpoints are involved:
Do all endpoints have to improve?
What if the model has near-perfect performance with respect to some endpoints and room for improvement for others?
To tackle these questions, we must run both superiority and non-inferiority (NI) tests.
Moreover, introducing NI tests prompts even more questions, such as how to choose an appropriate NI margin.

%Moreover, modifications are proposed in a sequential and adaptive manner, which can depend on the outputs from an aACP.
%An aACP that aims to control the online error rate must account for this adaptivity in order to avoid inflated error rates in adversarial setups.

%The gold standard can be measured accurately and quickly without major ethical issues.
%Also, the risk of this device is also relatively low since one can always use a gold-standard procedure instead.

%\url{https://techcrunch.com/2018/07/12/sight-diagnostics-launches-an-ai-based-diagnostics-device-for-faster-blood-tests/}
%\url{https://en.globes.co.il/en/article-sight-diagnostics-successful-in-clinical-trial-1001281957}
%\url{https://clinicaltrials.gov/ct2/show/NCT03595501}

\subsection{Detecting large vessel occlusion from CT angiogram images of the brain}
\label{sec:cta}
ContaCT is a SaMD that identifies whether CT angiogram images of the brain contain a suspected large vessel occlusion.
If so, it notifies a medical specialist to intervene.
The manufacturer evaluated ContaCT using images analyzed by neuro-radiologists.
The primary endpoints were estimated sensitivity and specificity.
The secondary endpoint was the difference in notification time between ContaCT and standard-of-care.
ContaCT achieved 87\% sensitivity and 89\% specificity and significantly shortened notification time.

Having obtained FDA approval \citep{FDA2018-dq}, the company might want to improve ContaCT by, say, training on more images, extracting a different set of image features, or utilizing clinical covariates from electronic health records.
This last modification type requires special consideration since the distribution of clinical covariates and their missingness distribution are susceptible to time trends.

%\url{https://www.accessdata.fda.gov/cdrh_docs/reviews/DEN170073.pdf}

\subsection{Blood test for cancer risk prediction}
\label{sec:grail}
GRAIL is designing a blood test that sequences cell-free nucleic acids (cfNAs) circulating in the blood to detect cancer early.
They are currently evaluating this test in an observational study (ClinicalTrials.gov ID NCT02889978) where the gold standard is a cancer diagnosis from the doctor within 30 months.
For time-varying outcomes, one may consider evaluating performance using time-dependent endpoints, such as those in \citet{Heagerty2005-ui}.

After the blood test is approved, GRAIL might still want to change their prediction algorithm.
For example, they could collect additional omics measurements, sequence the cfNAs at a different depth (e.g. lower to decrease costs, higher to improve accuracy), or train the model on more data.
Regulating modifications to this blood test is particularly difficult because the gold standard might not be observable in all patients, its definition can vary between doctors, and it cannot be measured instantaneously.
In fact, the gold standard might not be measurable at all because test results will likely affect patient and doctor behavior.

\section{Problem Setup}

In this section, we provide a general framework and abstractions to understand the approval process for modifications to AI/ML-based SaMD.
We begin with reviewing the approval process for a single AI/ML-based SaMD since it forms the basis of our understanding and is a prerequisite to getting modifications approved.

\subsection{AI/ML-based SaMD}

Formally, the FDA defines SaMD as software intended to be used for one or more medical purposes without being part of a hardware medical device.
An AI/ML-system is software that learns to perform a specific task by tracking performance measures.
The FDA approves a SaMD for a specific indication, which describes the population, disease, and intended use.
We only focus on SaMDs intended to be inform and are approved based on predictive accuracy, not those that prescribe treatment and are evaluated based on patient outcomes.
%This paper will only discuss regulating modifications to AI/ML-based SaMD that are evaluated based on predictive accuracy.
%We will not discuss modifications to SaMD that are assessed based on patient outcomes, which are common in treatment-coupled biomarkers.
%There are additional complications in this setting since counterfactual outcomes are typically unobserved.

Predictive accuracy is typically characterized by multiple endpoints, or co-primary endpoints \citep{Offen2007-dq, Fda2017-vs}.
The most common endpoints for binary classifiers are sensitivity and specificity because they tend to be independent of disease prevalence, which can vary across subpopulations \citep{Pepe2003-ri}.
Additionally, we can evaluate endpoints over different subgroups to guarantee a minimum level of accuracy for each one.
%Finally, multiple endpoints can capture the different aspects of the SaMD, such as its safety and efficacy.

%In general, we recommend using endpoints that are independent of prevalence since prevalence can differ across subpopulations. So for binary classification, we recommend using sensitivity and specificity rather than misclassification rates or positive predictive value.
%For binary outcomes, sensitivity and specificity are popular choices because they are independent of disease prevalence \citep{Pepe2003-ri}.
%Analogously, time-dependent sensitivity and specificity can be used for survival outcomes \citep{Heagerty2005-ui}.
%In order to ensure a minimum level of accuracy across different subpopulations, one might also consider endpoints evaluated over each subpopulation.

We now define a model developer (the manufacturer) in mathematical terms.
Let $\mathcal{X}$ be the support of the targeted patient population, where a patient is represented by their covariate measurements.
Let $\mathcal{Y}$ be output range (possibly multivariate).
Let $\mathcal{Q}$ be a family of prediction models $f:\mathcal{X} \mapsto \mathcal{Y}$.
Each model $f$ defines the entire pipeline for calculating the SaMD output, including feature extraction, pre-processing steps, and how missing data is handled.
The model developer is a functional $g$ that maps the training data $(X_T, Y_T) \in \mathcal{X}^n \times \mathcal{Y}^n$ to a function in $\mathcal{Q}$.
Let $\mathcal{P}$ be the family of distributions for $X \times Y$.
The performance of a model $f$ on population $\mathbb{P} \in \mathcal{P}$ is quantified by the $K$-dimensional endpoint $m: \mathcal{Q} \times \mathcal{P} \mapsto \mathbb{R}^{K}$.
For each endpoint $m_k$, we assume that a larger value indicates better performance.

\subsection{Modifications to AI/ML-based SaMD}
The proposed workflow in \citet{Fda2019-kt} for modifying an AI/ML-based SaMD iterates between three stages. % (Figure~\ref{fig:fda_tplc}).
First, the manufacturer proposes a modification by training on monitoring and/or external data and adds this to a pool of proposed modifications.
Second, the aACP evaluates each candidate modification and grants approval to those satisfying some criteria.
The most recently approved version is then recommended to doctors and patients.
Finally, monitoring data is collected, which can be used to evaluate and train future models.

Within this workflow, the model developer acts in a sequential and possibly adaptive manner.
For simplicity, consider a fixed grid of time points $t = 1,2,...$.
Since we allow arbitrary modifications, we treat each modification as an entirely separate model.
At each time point, the model developer proposes a new model and adds it to the pool of candidates.
(For example, they may submit a new model trained on monitoring data obtained  at the end of each month.)
Let filtration $\mathcal{F}_t$ be the sigma algebra representing the information up to time $t$, which includes observed monitoring data, proposed models, and aACP outputs up to time $t$.
The model developer is a sequence of functionals $\{g_t: t = 1,2,...\}$, where $g_t$ is a $\mathcal{F}_t$-measurable functional mapping to $\mathcal{Q}$.
Let $\hat{f}_t$ be the realized model proposal at time $t$.
In addition, suppose that each proposed model $\hat{f}_t$ has a maximum wait time $\Delta_t$ that specifies how long the manufacturer will wait for approval of this model, i.e. the model is no longer considered for approval after time $t + \Delta_t$.
%The maximum wait time reflects our assumption that the manufacturer is uninterested in getting approval for old unapproved modifications.

Time trends are likely to occur in long-running processes, as found in long-running clinical trials and non-inferiority trials \citep{Altman1988-ax, Fleming2008-ok}.
This includes changes to any component of the joint distribution between the patient population and the outcome, such as the marginal distributions of the covariates, their correlation structure, their prognostic values, and the prevalence of the condition.
As such, let the joint distribution at time $t$ of patients $X_t$ and outcomes $Y_t$ be denoted $\mathbb{P}_t$.
The value of endpoint $m$ for model $f$ at time $t$ is then $m(f, \mathbb{P}_{t})$.
More generally, we might characterize a model at time $t$ by the average endpoint value over the previous $D$ time points, as evaluated by $m(f, \mathbb{P}_{t:t + D - 1})$ where $\mathbb{P}_{t:t'}$ indicates a uniform mixture of $\mathbb{P}_t, ..., \mathbb{P}_{t'}$.
Here $D$ acts as a smoothing parameter; Larger $D$ increases the smoothness of endpoint values.

Finally, this paper assumes that monitoring data collected at time $t$ are representative of the current population $\mathbb{P}_t$.
Of course, satisfying this criteria is itself a complex issue.
We will not discuss the challenges here and instead refer the reader to \citet{Pepe2003-ri} for more details, such as selecting an appropriate sampling scheme, measuring positive versus negative examples, and obtaining gold standard versus noisy labels.

\subsubsection{Defining acceptable modifications}

\begin{figure}
	\centering
	\includegraphics[width=0.75\linewidth]{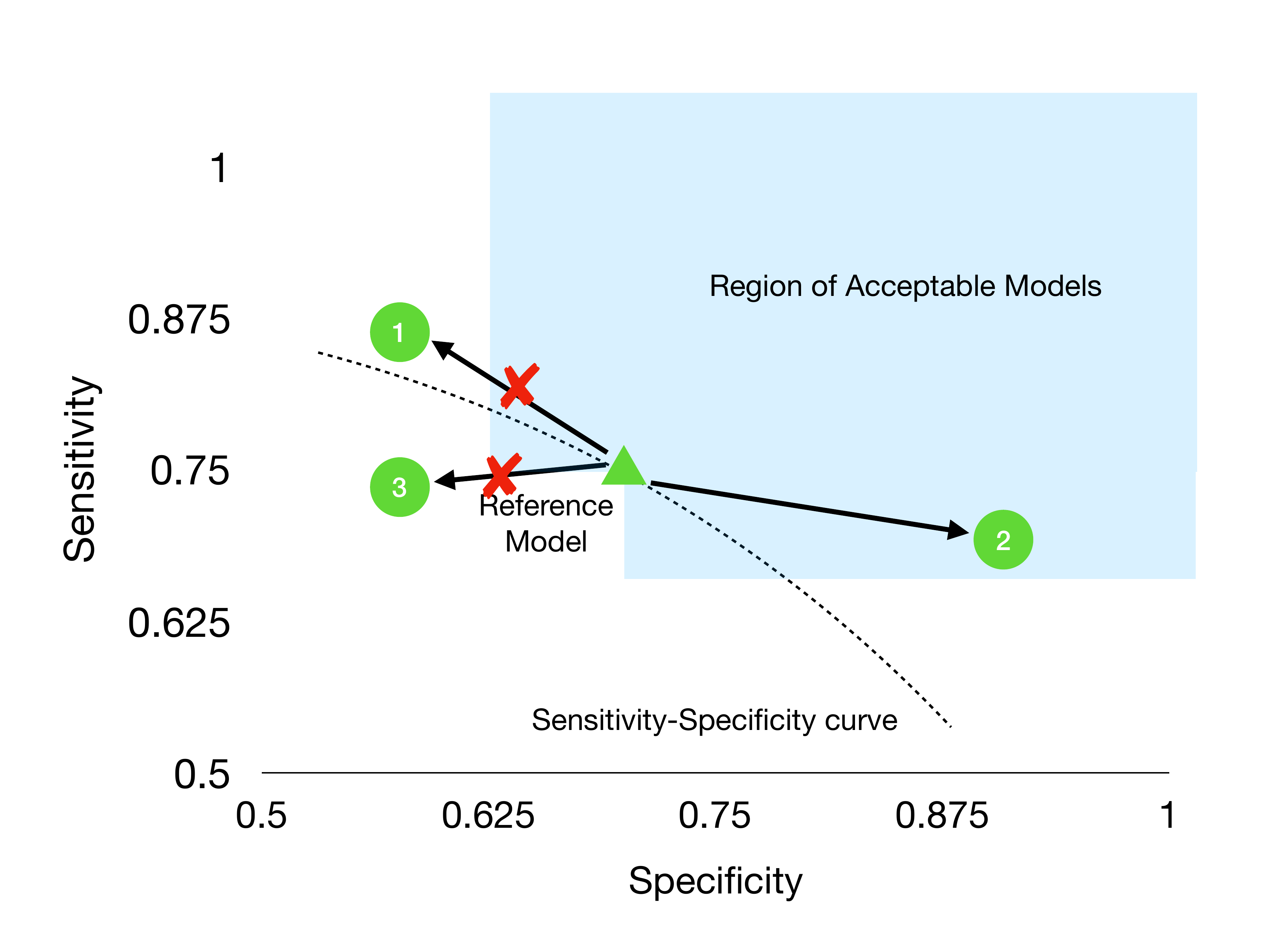}
	\caption{
	Example of an acceptability graph for binary classifiers evaluated on sensitivity and specificity.
	Given a reference model (triangle) and NI margin $\epsilon$, a candidate model is acceptable if one endpoint is non-inferior and the other is superior compared to the reference model.
	The NI margin can be chosen to encourage approval of updates to a better ROC curve.
	Models in the shaded blue area are acceptable updates to the reference model.
	Model 3 is not acceptable since it is on a strictly inferior ROC curve.
	Model 1 and 2 are likely on better ROC curves, but 1 is not within the NI margin and is therefore not acceptable either.
	}
	\label{fig:roc}
\end{figure}
A fundamental building block for designing an aACP is defining when a modification is acceptable to a reference model.
Our solution is to represent which modifications are acceptable using a directed graph between models in $\mathcal{Q}$.
If there is a directed edge from model $f$ to model $f'$, then it is acceptable to update $f$ to $f'$.
This ``acceptability graph'' is parameterized by a pre-defined vector of non-inferiority margins $\epsilon \in \mathbb{R}^K_+$.
An update from $f$ to $f'$ is acceptable if it demonstrates non-inferiority with respect to all endpoints and superiority in at least one \citep{Bloch2001-br, Bloch2007-xk}.
So for a binary classifier where the endpoints are sensitivity and specificity, one may select the NI margins to encourage modifications that shift the model to a better ROC curve (Figure~\ref{fig:roc}).
An acceptability graph is formally defined below:
\begin{definition}
	For a fixed evaluation window $D \in \mathbb{Z}^+$ and NI margin $\epsilon \in \mathbb{R}^K_+$, the acceptability graph at time $t$ over $\mathcal{Q}$ contains the edge from $f$ to $f'$ if $m_k(f, \mathbb{P}_{t:t + D - 1}) - \epsilon_k \le m_k(f', \mathbb{P}_{t:t + D - 1})$ for all $k = 1,...,K$ and there is some $k = 1,...,K$ such that $m_k(f', \mathbb{P}_{t:t + D - 1}) > m_k(f, \mathbb{P}_{t:t + D - 1})$.
	The existence of this edge is denoted $f \rightarrow_{\epsilon, D, t} f'$ and $f \nrightarrow_{\epsilon, D, t} f'$ otherwise.
\label{eq:acc_graph}
\end{definition}
In this paper, we assume $D$ is fixed and use the notation $f \rightarrow_{\epsilon, t} f'$.
For simplicity, Definition~\ref{eq:acc_graph} uses the same NI margin across all models.
In practice, it may be useful to let the margin depend on the reference model or the previously established limits of its predictive accuracy.

We obtain different graphs for different choices of $\epsilon$.
For instance, $\epsilon = 0$ means that a model is only acceptable if it is superior with respect to all endpoints, though this can be overly strict in some scenarios.
%For example, it will be difficult to demonstrate superiority if a binary classifier already has 99\% specificity; Instead, the manufacturer likely wants to keep specificity constant while improving sensitivity.
Setting $\epsilon \ne 0$ is useful for approving modifications that maintain the value of some endpoints or have very small improvements with respect to some endpoints.

Finally, we define hypothesis tests based on the acceptability graph.
In an $\epsilon$-acceptability test, we test the null hypothesis is that a model $f'$ is not an $\epsilon$-acceptable update to model $f$ at time $t$, i.e. $H_0: f \nrightarrow_{\epsilon, t} f'$ .
A superiority test is simply an $\epsilon$-acceptability test where $\epsilon = 0$.

\section{An online hypothesis testing framework}

At each time point, we suppose an aACP evaluates which candidates to approve by running a battery of hypothesis tests.
As such, we frame an aACP as an online hypothesis testing procedure.
In contrast to one-time hypothesis tests, online hypothesis testing procedures aim to control the error rate over a sequence of tests.
Accounting for the multiplicity of hypotheses is important since new modifications to an AI/ML-based SaMD can be proposed more easily and frequently compared to the drug development setting.

Each aACP specifies a sequence of approval functions $A_t$ for times $t = 1,2,...$ (Figure~\ref{fig:aacp}), where $A_t$ is a $\tilde{\mathcal{F}}_t$-measurable function that outputs the index of the most recently approved model at time $t$ (some value in $\{0,...,t - 1\}$).
Filtration $\tilde{\mathcal{F}}_t$ is the sigma-algebra for monitoring data up to time $t$ and proposed models and aACP outputs up to time $t - 1$.
The index of the latest approved model at time $t$ is denoted $\hat{A}_t$.
A model was approved at time $t$ if $\hat{A}_t \ne \hat{A}_{t-1}$.
Assuming companies are not interested in approving older models, we require $\hat{A}_t \ge \hat{A}_{t-1}$.

\begin{figure}
	\centering
	\includegraphics[width=\linewidth]{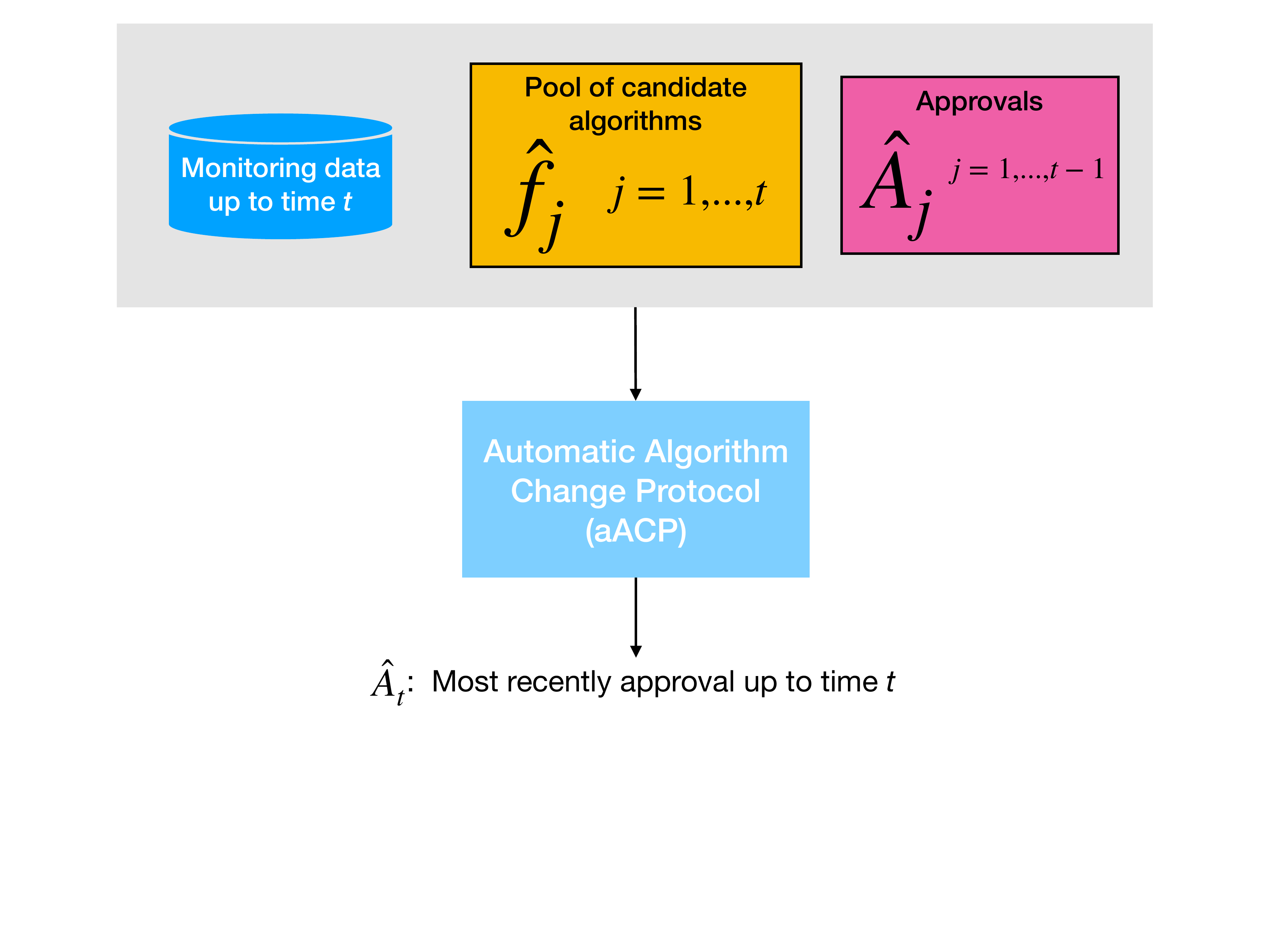}
	\vspace{-0.8in}
	\caption{
	An automatic Algorithm Change Protocol (aACP) outputs the index of the most recently approved model $\hat{A}_t$ at each time $t$.
	To do so, it evaluates the pool of candidate models against the pool of previously approved models using monitoring data collected up to that time.
	}
	\label{fig:aacp}
\end{figure}

Different approval functions lead to different aACPs.
%We obtain different aACPs by varying the definitions of the approval functions.
For example, the following are two simple aACPs that one may plausibly consider but do not provide error-rate guarantees.

\begin{quote}
\textbf{aACP-Baseline} approves any modification that demonstrates $\epsilon$-acceptability to the initially approved model at a fixed level $\alpha$.
This can be useful when the initial model has high predictive accuracy.
The manufacturer may also argue this is reasonable policy because the current laws only require a model to perform better than placebo, i.e. the standard of care without utilizing AI/ML-based SaMDs.
\end{quote}
\begin{quote}
\textbf{aACP-Reset} approves any modification that demonstrates $\epsilon$-acceptability to the currently approved model at some fixed level $\alpha$.
As opposed to aACP-Baseline, this policy encourages the model to improve over time.
\end{quote}

%To ensure the quality of approved models, an aACP should try to control the frequency of erroneous approvals.
%Since the most recently approved model is recommended for wide deployment, an unacceptable modification is considered an error.
%More generally, an aACP that makes $d$ decisions at each time point should control at least $d$ different online error rates to ensure the quality of its decisions.

\section{Online error rates for aACPs}

We define two online Type I error rates that an aACP might try to control and describe aACPs that uniformly control the error rates over time.
Manufacturers and regulators should select the error rate definition and aACP most suitable for their purposes.
These aACPs achieve error rate control as long as their individual hypothesis tests are controlled at their nominal levels.
We achieve this by testing on only prospectively-collected monitoring data.

For both definitions, the error rate at time $T$ is evaluated over a window of width $W$, i.e. time points $1 \vee (T - W)$ to $T$.
The hyperparameter $W$ must be pre-specified and specifies different trade-offs between error control and speed: $W =\infty$ requires the strongest error rate control, but is overly strict in most cases, and $W = 1$ requires the weakest error control, but can lead to bad long-running behavior.
The desired trade-off is typically in between these extremes.
%We will not consider procedures that sequester a portion of data for sole use by the aACP and re-use data across hypothesis tests.
%Though such procedures may be more efficient in certain cases, they must carefully take into account the adaptive nature of the proposed models to minimize over-optimism, perhaps by using ideas from differential privacy \citep{Dwork2015-da, Blum2015-hv}.

\subsection{Bad approval count}
\label{sec:bad_prob}
We define a bad approval as one where the modification is unacceptable with regards to \textit{any} of the previously approved models.
The first error rate is defined as the expected Bad Approval Count (BAC) within the current window of width $W$:
%Based on this, the first error rate is defined as the expected Bad Approval Count (BAC) within the current window of width $W$, as given below:
\begin{definition}
	The expected bad approval count within the $W$-window at time $T$ is
	$$
	\BAP_W(T) = E \left[
	\sum_{t = 1 \vee (T - W)}^{T}
	\mathbbm{1}
	\left\{
	\exists t' = 1,...,t-1
	\text{ s.t. }
	\hat{f}_{\hat{A}_{t'}}
	\nrightarrow_{\epsilon, t}
	\hat{f}_{\hat{A}_{t}}
	\right\}
	\right].
	$$
	\label{def:fwer_window}
\end{definition}
This error rate captures two important ways errors can accumulate over time: bio-creep and the multiplicity of hypotheses.
We discuss these two issues below.

When a sequence of NI trials is performed and the reference in each trial is the latest model that demonstrated NI, the performance of the approved models will gradually degrade over time; This phenomenon has been called bio-creep in previous work \citep{Fleming2008-ok}.
Bio-creep can also happen in our setting: Even if each approved model demonstrates superiority with respect to some endpoints and NI with respect to others, repeated applications of $\epsilon$-acceptability tests can still lead to approval of strictly inferior models.
The risk of bio-creep is particularly pronounced because the model developer can perform unblinded adaptations.
To protect against bio-creep, Definition~\ref{def:fwer_window} counts it as a type of bad approval.
%So, an aACP that controls the expected bad approval count at some level $\alpha$ also controls the probability of bio-creep (within that window) at level $\alpha$.

Second, when a long sequence of hypothesis tests is performed, the probability of a false rejection is inflated due to the multiplicity of hypotheses.
Definition~\ref{def:fwer_window} accounts for multiplicity by summing the probabilities of bad approvals across the window.
It is an upper bound for the probability of making any bad approval within the window, which is similar to the definition of family-wise error rate (FWER).
In fact, we use the connection between FWER and BAC in the following section to design an aACP that controls this error rate.
%In fact, the following section proposes an aACP that essentially uses Bonferroni connection to control $\BAP_W$, but properly accounts for unblinded adaptation by launching fixed hypothesis tests and only uses future monitoring data.

\subsubsection{aACP to control bad approval counts}
\label{sec:bad_prob_acp}

We now present aACP-BAC, which uniformly controls $\BAP_W(\cdot)$.
An aACP is defined by its skeletal structure, which specifies the sequence of hypothesis tests run, and a procedure that selects the levels to perform the hypothesis tests.
To build up to aACP-BAC, we i) first describe a simple aACP skeleton that launches a fixed sequence of group sequential tests (GSTs), ii) add gate-keeping to increase its flexibility, and iii) finally pair it with a sequence of $\tilde{\mathcal{F}}_t$-measurable functions $\{\alpha_t: t = 1,2,...\}$ for choosing the hypothesis test levels.
The full algorithm is given in Algorithm~\ref{algo:aacp_count} in the Appendix.
For now, we assume the distributions are constant and simply use the notation $\rightarrow_\epsilon$ in place of $\rightarrow_{\epsilon, t}$.
We discuss robustness to time trends in a later section.

\begin{figure}
%	\centering
	\includegraphics[width=1.2\linewidth]{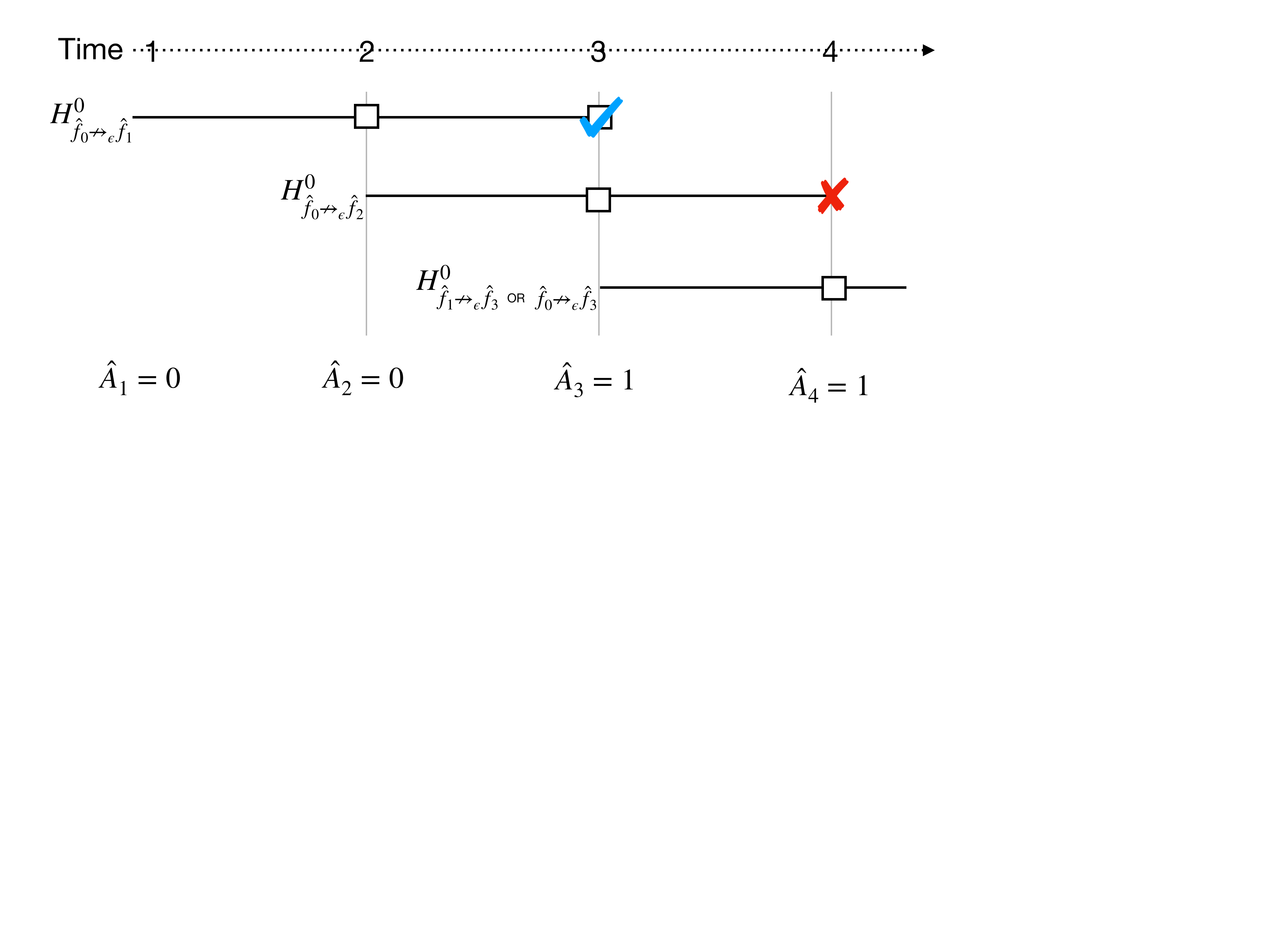}
	\vspace{-1.8in}
	\caption{
	At each time point, this simple aACP launches a single group sequential test (GST) comparing the newly proposed model to previously approved models.
	Here, each model has a maximum wait time of $\Delta = 2$ and each interim analysis is represented by a square.
	A checkmark indicates that the null hypothesis is rejected and an ``X'' indicates that the interim analysis is not performed.
	The final interim analysis for $\hat{f}_2$ is not performed because its GST only compares $\hat{f}_2$ to $\hat{f}_0$ and not the newly approved model $\hat{f}_1$.
	Thus, $\hat{f}_2$ has no chance of being approved.
	}
	\label{fig:aacp_simple}
\end{figure}

Let us first consider the a simple aACP skeleton that compares each proposed model to previously approved models using a single hypothesis test (Figure~\ref{fig:aacp_simple}).
More specifically, at time $t$, it launches a group sequential $\epsilon$-acceptability test with the null hypothesis
\begin{align}
H_0: \exists t' = 1,...,t \text{ s.t. } \hat{f}_{\hat{A}_{t'}} \nrightarrow_{\epsilon} \hat{f}_t.
\label{eq:simple_null}
\end{align}
The number of interim analyses is the maximum wait time $\Delta_t$ and the critical values are chosen according to an alpha-spending function specified prior to launch \citep{DeMets1994-pf}.
At each time point, we also perform interim analyses for all active hypothesis tests (i.e. those not past their maximum wait time).
The aACP approves $\hat{f}_j$ at time $t$ if it demonstrates acceptability to $\hat{f}_{\hat{A}_1}, ...,\hat{f}_{\hat{A}_{t-1}}$.
If multiple models are acceptable, it selects the latest one.

%This simple aACP skeleton evaluates each proposed model by launching an individual hypothesis test (Figure~\ref{fig:aacp_simple}).
%It performs two steps at each time $t$.
%First, it launches a level-$\hat{\alpha}_t$ group sequential $\epsilon$-acceptability test with the null hypothesis that there is some $t' = 1,...,t$ such that $\hat{f}_{\hat{A}_{t'}} \nrightarrow_{\epsilon} \hat{f}_t$.
%The maximum wait time for this test is $\Delta_t$ and the aACP must select a corresponding alpha-spending function \citep{DeMets1994-pf}.
%Second, the aACP performs interim analyses for all active hypothesis tests (i.e. those that are not past their maximum wait time).
%If an active test rejects the null hypothesis and demonstrates model $\hat{f}_j$ is acceptable to models $\hat{f}_{\hat{A}_{1}},...,\hat{f}_{\hat{A}_{t - 1}}$, then the aACP approves $\hat{f}_j$, i.e. $\hat{A}_t = j$.
%If multiple models are acceptable, the aACP selects the latest one.

A drawback of this simple aACP skeleton is that it fails to adapt to new model approvals that occur in the middle of a group sequential test (GST).
Consider the example in Figure~\ref{fig:aacp_simple}, where a GST with null hypothesis $H^0_{\hat{f}_0 \nrightarrow_{\epsilon} \hat{f}_1}$ is launched at time $t = 1$ and a second GST with null hypothesis $H^0_{\hat{f}_0 \nrightarrow_{\epsilon} \hat{f}_2}$ is launched at time $t = 2$.
If $\hat{f}_1$ is approved at time $t = 3$, this aACP cannot approve $\hat{f}_2$ since its GST only compares $\hat{f}_2$ to $\hat{f}_0$.
Ideally, it could adapt to the new approval and add a test comparing $\hat{f}_2$ to $\hat{f}_1$.

aACP-BAC addresses this issue by evaluating proposed model $\hat{f}_t$ using a \textit{family} of acceptability tests instead (Figure~\ref{fig:aacp_simple}).
In addition to the aforementioned test for the null hypothesis \eqref{eq:simple_null}, this family includes acceptability tests to test each of the null hypotheses
\begin{align}
H_{0,j}: \hat{f}_{j} \nrightarrow \hat{f}_t \text{ for } j = \hat{A}_t + 1, ..., t - 1.
\end{align}
As before, a model is approved at time $t$ only if it demonstrates acceptability compared to all approved models up to time $t$.
To control the online error rate, aACP-BAC controls the FWER for each family of tests using a serial gate-keeping procedure.
Recall that gate-keeping tests hypotheses in a pre-specified order and stops once it fails to reject a null hypothesis \citep{Dmitrienko2007-hp}.
No alpha adjustment is needed in gate-keeping; It controls FWER at $\alpha$ by performing all tests at level $\alpha$.
Here, the tests are naturally ordered by the index of the reference models, from oldest to latest.
Moreover, this ordering maximizes the probability of approval, assuming the proposed models improve in predictive accuracy.
Details for performing GSTs with gate-keeping are given in \citet{Tamhane2018-rr}.

%We now describe the full skeleton of aACP-BAC (Algorithm~\ref{algo:aacp_count}).
%As before, aACP-BAC performs two steps at each time $t$.
%First, it launches the aforementioned family of GSTs comparing models against the newly proposed model $\hat{f}_t$ and test the hypotheses at level $\hat{\alpha}_t$ using gatekeeping \citep{Tamhane2018-rr}.
%\textred{do we need to talk about how gatekeeping + GST + multiple endpoints is done exactly? does the previous reference work?}
%The aACP must specify the maximum wait time for the family $\Delta_t$ and the alpha-spending function.
%Second, it performs all interim analyses for all active tests using gatekeeping.
%Model $\hat{f}_{t'}$ is approved at time $t$ if it demonstrates acceptability compared to models $\hat{f}_{\hat{A}_1},...,\hat{f}_{\hat{A}_{t'}},\hat{f}_{\hat{A}_{t'} + 1},...,\hat{f}_{\hat{A}_{t}}$.

To uniformly control $\BAP_W(\cdot)$ at $\alpha$, aACP-BAC computes an over-estimate of $\BAP_W(t)$ at each time $t$ and selects level $\hat{\alpha}_t$ such that the over-estimate is bounded by $\alpha$.
Using a union bound like that in Bonferroni correction, it uses the over-estimate
\begin{align}
\widehat{\BAP}_W(t) =
\sum_{t' = 1}^{t} \hat{\alpha}_{t'} \mathbbm{1}\left\{
t  - W \le t' + \Delta_{t'} \le t
\right\}
\label{eq:bap_estimate}
\end{align}
and selects $\hat{\alpha}_t$ such that
\begin{align}
\widehat{\BAP}_W(t) \le \alpha.
\label{eq:bap_control}
\end{align}
In the Appendix, we prove that aACP-BAC achieves the nominal rate.

Alternatively, we can think of aACP-BAC as an alpha-investing procedure \citep{Foster2008-ek} that begins with an alpha-wealth of $\alpha$, spends it when a family of tests is launched, and earns it back when the family leaves the current window.
From this, we can see that choosing an infinitely long window ($W = \infty$) has low power because the aACP will spend but never earn alpha-wealth.
This is analogous to the so-called ``alpha-death'' issue that occurs in procedures that control online FWER \citep{Ramdas2017-un}.
We sidestep the issue of alpha-death by selecting a reasonable value for $W$.

\begin{figure}
	\centering
	\includegraphics[width=1.2\linewidth]{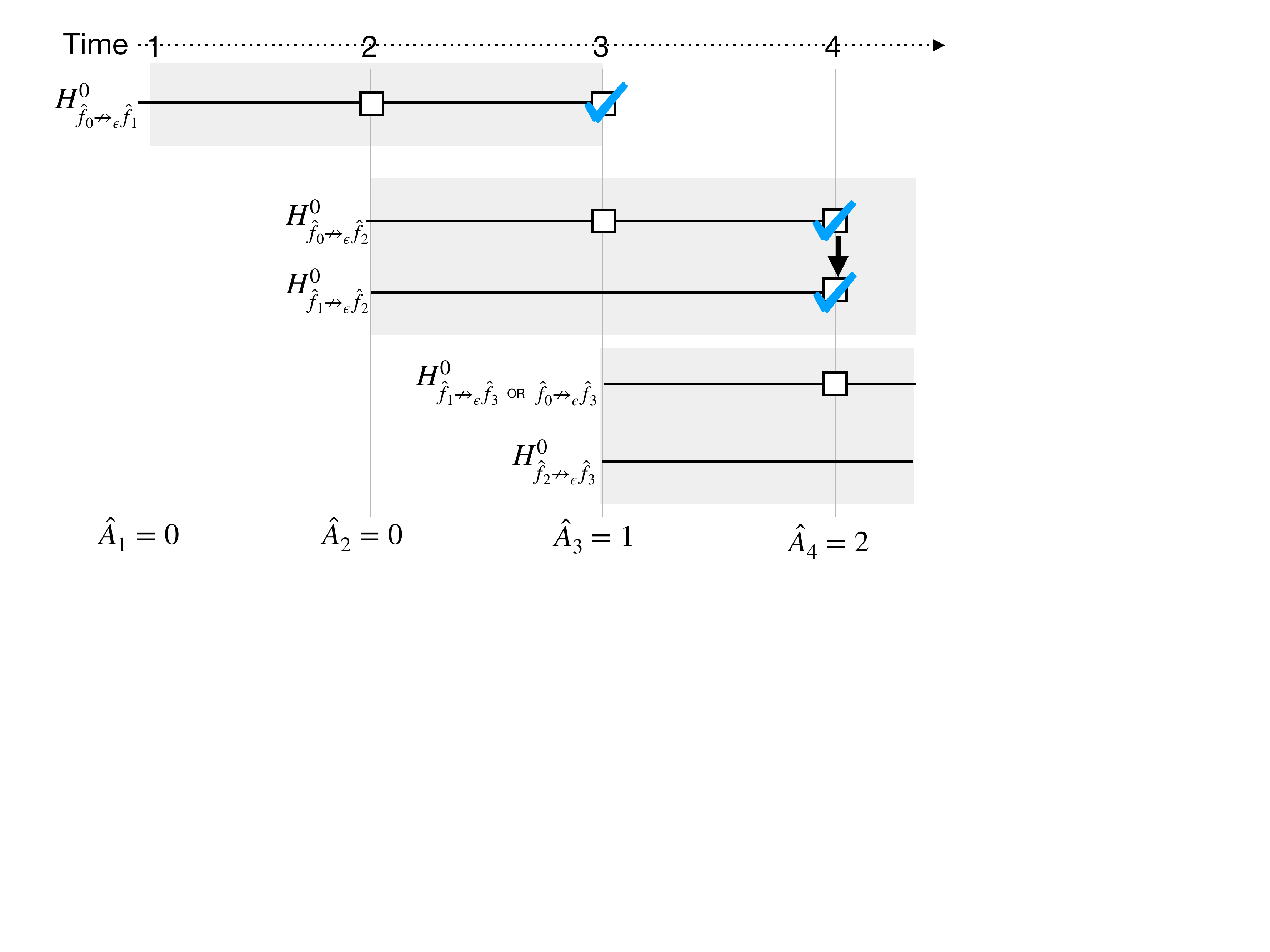}
	\vspace{-1.3in}
	\caption{
		At each time point, this aACP launches a family of group sequential tests (shaded gray boxes) comparing the newly proposed model to previously approved models as well as other models that might be approved in the interim.
		Within each family, we test the hypotheses using a gatekeeping procedure, which provides a mechanism for comparing a candidate model to newly approved models in the interim.
		We use the same notation in Figure~\ref{fig:aacp_simple}.
		An arrow between squares indicates that we rejected a null hypothesis and proceeded to the next test in the gatekeeping sequence.
	}
	\label{fig:aacp_skelly}
\end{figure}

\subsection{Bad approval and benchmark ratios}
\label{sec:bad_ratio}

If the goal is to ensure that the SaMD improves on average and occasional drops in performance are tolerated, the approval policies for controlling $\BAP_W$ can be overly strict and unnecessarily conservative.
There are two solutions to this problem.
One approach (\textit{reward-approach}) is to reward the company for each superior model by resetting the level alpha.
The FDA essentially uses this procedure right now, as each clinical trial resets the alpha-spending clock.
Another idea (\textit{FDR-approach}) is to draw on the false discovery rate (FDR) literature: These procedures control the expected proportion of false rejections rather than the FWER, which has higher power when  some of the null hypotheses are false \citep{Benjamini1995-xt}.
This section defines a second online error rate based on these ideas.

In line with the \textit{reward-approach}, we consider aACPs that utilize approval and ``benchmark'' functions to label models.
Whereas an approved model should be acceptable to previous approvals, a benchmark model should 1) be a previously-approved model and 2) be \textit{superior} to the previous benchmark.
A benchmark function $B_t$ is formally defined as a $\tilde{F}_t$ measurable function that outputs the index of the latest benchmark model at time $t$.
For $t = 0$, we have $B_0 \equiv 0$.
Again, we use the hat notation to indicate the realized benchmark index.
A bad benchmark is one in which $\hat{f}_{\hat{B}_{t - 1}} \nrightarrow_{0, t} \hat{f}_{\hat{B}_{t}}$.
We do not compare against all previous benchmarks since $\nrightarrow_{0, t}$ is a transitive property when the superiority graph is constant.

Based on the \textit{FDR-approach}, we now introduce bad approval and benchmark ratios.
An aACP needs to control both ratios to control the frequency of bad approvals and benchmarks.
\begin{definition}
	For NI margin $\epsilon$, the bad approval ratio within $W$-window at time $T$ is
	\begin{equation}
	\BAR_W(T)
	=
	\frac{
		\sum_{t =1 \vee (T - W)}^T \mathbbm{1}\left\{
		\exists t' = 1,...,t - 1 \text{ s.t. }
		\hat{f}_{\hat{A}_{t'}}
		\nrightarrow_{\epsilon, t}
		\hat{f}_{\hat{A}_{t}}
		\right\}
	}{
		1 + \sum_{t =1 \vee (T - W)}^T \mathbbm{1}\left\{\hat{B}_{t} \ne \hat{B}_{t - 1}\right \}
	}.
	\label{eq:bar}
	\end{equation}
	The bad benchmark ratio within $W$-window at time $T$ is
	\begin{equation}
	\BSR_W(T)
	=\frac{
		\sum_{t =1 \vee (T - W)}^T \mathbbm{1}\left\{
		\hat{f}_{\hat{B}_{t - 1}}
		\nrightarrow_{0, t}
		\hat{f}_{\hat{B}_{t}}
		\right\}
	}{
		1 + \sum_{t = 1 \vee (T - W)}^T \mathbbm{1}\left\{\hat{B}_{t} \ne \hat{B}_{t - 1}\right \}
	}.
	\label{eq:bsr}
	\end{equation}
\end{definition}
Since only approved models can be designated as benchmarks, $\BAR_{W}$ is an upper bound for the proportion of bad approvals (this is approximate because the denominator is off by one).
%Thus, controlling these two ratios is particularly suitable when we want the model to improve on average and allow occasional drops in performance.
So if a modification always decreases or increases performance by a single unit, $\BAR_{W}(\cdot) < 0.5$ means that the currently approved model is no worse than the initial version.
%In contrast, controlling the expected bad approval count $\BAP_W$ is more suitable when we would like to avoid any drops in performance with high probability.

The denominator in \eqref{eq:bar} was deliberately chosen to be the number of unique benchmarks rather than the number of approvals because the latter is easy to inflate artificially.
We can simply propose models by alternating between two models that are $\epsilon$-acceptable to each other.
This strategy does not work for benchmarks because they require demonstrating superiority.

%Again, this definition depends on the choice of the window hyperparameter.
%Not only does a finite window prevent the alpha-death issue, it also avoids an issue called piggy-backing, in which initially-collected alpha wealth allows one to reject many later hypotheses at large alpha levels.
%Previous work addresses these two issues by introducing an exponential decay term in the numerator and denominator, to obtain a regret-like term like that in reinforcement and online learning \citep{Ramdas2017-un}.
%However, we prefer a window-based definition since it is easy to interpret and control in settings where hypothesis tests are launched asynchronously.

\subsubsection{aACP to control bad approval and benchmark ratios}
\label{sec:bad_ratio_acp}

Instead of controlling the expectations of \eqref{eq:bar} and \eqref{eq:bsr}, we describe aACP-BABR for controlling the modified expected bad approval and benchmark ratios.
These modified ratios are based on a similar quantity in the online FDR literature known as modified online FDR \citep{Foster2008-ek}.
We chose to control the modified versions because they can be controlled under less restrictive conditions and using relatively intuitive techniques \citep{Ramdas2017-un}.
Moreover, \citet{Foster2008-ek} found that modified online FDR has similar long-running behavior to online FDR.
We define modified expected bad approval and benchmark ratios below.
%The online FDR literature considers a similar quantity known as \textit{modified} online FDR \citep{Foster2008-ek}.
%The advantage is that the modified version can be controlled under less restrictive conditions than online FDR, the techniques are generally more intuitive, and online FDR and modified online FDR have similar behaviors in long-running procedures \citep{Ramdas2018-gh, Zrnic2018-ab}.
\begin{definition}
	For NI margin $\epsilon$, the modified expected bad approval ratio within $W$-window at time $T$ is
	\begin{align}
	\begin{split}
	& \meBAR_W(T)\\
	& =
	\frac{
		E\left[
		\sum_{t =1 \vee (T - W)}^T \mathbbm{1}\left\{
		\exists t' = 1,...,t - 1 \text{ s.t. }
		\hat{f}_{\hat{A}_{t'}}
		\nrightarrow_{\epsilon, t}
		\hat{f}_{\hat{A}_{t}}
		\right\}
		\right]
	}{
		E\left[
		1 +
		\sum_{t =1 \vee (T - W)}^T \mathbbm{1}\left\{\hat{B}_{t} \ne \hat{B}_{t - 1}\right \}
		\right]
	}.
	\label{eq:bar_decay}
	\end{split}
	\end{align}
	The modified expected bad benchmark ratio within $W$-window at time $T$ is
	\begin{equation}
	\meBSR_{W}(T)
	=\frac{
		E\left[\sum_{t = 1 \vee (T - W)}^T
		\mathbbm{1}\left\{
		\hat{f}_{\hat{B}_{t - 1}}
		\nrightarrow_{0, t}
		\hat{f}_{\hat{B}_{t}}
		\right\}
		\right]
	}{
		E\left[1 + \sum_{t = 1 \vee (T - W)}^T
		\mathbbm{1}\left\{\hat{B}_{t} \ne \hat{B}_{t - 1}\right \}
		\right]
	}.
	\end{equation}
\end{definition}

\noindent Next, we describe how aACP-BABR uniformly controls $\meBAR_W(\cdot)$ and $\meBSR_{W}(\cdot)$ at levels $\alpha$ and $\alpha'$, respectively (Algorithm~\ref{algo:aacp_ratio}).
We begin with its skeleton and then discuss the alpha-investing procedure.
Again, we assume the distributions $\mathbb{P}_t$ are constant.

aACP-BABR uses the acceptability tests from aACP-BAC to approve modifications and superiority tests to discover benchmarks.
So at time $t$, in addition to launching a family of acceptability tests to evaluate model $\hat{f}_t$ for approval, aACP-BABR also launches a family of group-sequential superiority tests comparing $\hat{f}_t$ to models with indices $\left\{\hat{B}_{t - 1}, .... , t - 1 \right\}$, which are executed in a gate-keeping fashion from oldest to latest.
Let $\Delta_t'$ be the maximum wait time for the superiority tests, which can differ from the maximum wait time for acceptability tests.
%the maximum wait times and alpha-spending functions can differ between the acceptability and superiority tests.
A model $\hat{f}_j$ is designated as a new benchmark at time $t$ if it demonstrates superiority to models $\hat{f}_{\hat{B}_{j - 1}},...,\hat{f}_{\hat{B}_{t - 1}}$.
If multiple benchmarks are discovered at the same time, the aACP can choose any of them (we choose the oldest one in our implementation).

aACP-BABR uses an alpha-investing procedure based on \citet{Ramdas2017-un} to control the error rates.
Let the $\tilde{F}_t$-measurable function $\alpha_t'$ specify the level to perform superiority tests launched at time $t$.
At time $t$, aACP-BABR constructs over-estimates of the error rates $\BAR_{W}(t)$ and $\BSR_W(t)$ and selects $\hat{\alpha}_t$ and $\hat{\alpha}_{t'}$ such that the over-estimates are no larger than the nominal levels.
The over-estimates are
\begin{align}
\widehat{\BAR}_{W}(t)
& =
\frac{
	\sum_{t' = 1}^t
	\hat{\alpha}_{t'} \mathbbm{1}\left\{
	t  - W \le t' + \Delta_{t'} \le t
	\right\}
}{
	1 +
	\sum_{t' = 1\vee (t-W)}^t
	\mathbbm{1}\{\hat{B}_{t'} \ne \hat{B}_{t' - 1} \}
}
\label{eq:est_bar}\\
\widehat{\BSR}_{W}(t)
& =
\frac{
	\sum_{t' = 1}^t
	\hat{\alpha}'_{t'}
	\mathbbm{1}\left\{
	t  - W \le t' + \Delta'_{t'} \le t
	\right\}
}{
	1 +
	\sum_{t' = 1\vee (t-W)}^t
	\mathbbm{1}\{\hat{B}_{t'} \ne \hat{B}_{t' - 1} \}
}.
\label{eq:est_bsr}
\end{align}
It selects $\hat{\alpha}_t$ and $\hat{\alpha}'_t$ such that
\begin{align}
\widehat{\BAR}_{j}(t) \le \alpha& \quad \forall j = 1,...,W
\label{eq:bar_control}
\\
\widehat{\BSR}_{j}(t) \le \alpha'& \quad \forall j = 1,...,W.
\label{eq:bsr_control}
\end{align}
(We consider all window sizes since we also need to over-estimate future errors $\BAR_{W}(t')$ and $\BSR_W(t')$ for $t' > t$.)
%\textred{(we actually have to consider all possible windows, to account for a worst case scenario where no benchmarks are approved in the upcoming time points)}
So, aACP-BABR earns alpha-wealth when new benchmarks are discovered, which unites ideas from \textit{FDR-approach} and \textit{reward-approach}.
We prove in the Appendix that aACP-BABR provides the desired error control.

%The aACP we have proposed should be considered a first pass.
%There are many alternative designs we have not pursued but are interesting lines of work.
%One idea is to launch superiority tests for a model only after it is approved.
%Though this discards the data used to demonstrate acceptability, it spends the alpha adaptively based on promising models.
%If we would like to reuse data, we can also explore ways to reuse the data used to demonstrate acceptability using differential privacy ideas.
%Another idea is to optimize the alpha-spending functions for the group-sequential superiority and acceptability tests.
%In our current implementation, we have not optimized either and use the same alpha-spending families for both.

\subsection{Effect of time trends}
\label{sec:time_trends}
Time trends are likely to occur when an aACP is run for a long time.
We now discuss how robust aACP-BAC and -BABR are to time trends.
We consider levels of increasing severity: the distributions are relatively constant over time (\textit{no-trend}), the distributions are variable but the acceptability graphs are relatively constant (\textit{graph-constant}), and the acceptability graphs change frequently (\textit{graph-changing}).

When the distributions are relatively constant over time, aACP-BAC and -BABR should approximately achieve their nominal error rates.
Recall that the two aACPs perform paired T-tests by approximating the distribution $\mathbb{P}_{t:t + D - 1}$ with monitoring data sampled from $\mathbb{P}_{j \vee j':t - 1}$, which is reasonable when the distributions are relatively constant over time.
When there are multiple endpoints, one can either choose a GST that rejects the null hypothesis when all endpoints surpass the significance threshold at the same interim time point or when the endpoints surpass their respective thresholds at any interim timepoint.
The former approach is more robust to time trends with only modest differences in power \citep{Asakura2014-oz}.
%Thus the aACPs test the time-dependent null hypotheses $H^0: \hat{f}_{j} \nrightarrow_{\epsilon, t} \hat{f}_{j'}$ at approximately the nominal Type I error rates, which means their error rates should be close to the nominal rates.

When the distributions are not constant but the acceptability graphs are, the GSTs have inflated error rates since they only guarantee Type I error control under the strong null.
%This can lead to inflated error rates of the aACP-BAC and aACP-BABR.
%\textred{are there group sequential methods for significant heterogeneity across batches?}
To handle heterogeneity in distributions over time, we can instead use combination tests, such as Fisher's product test and the inverse normal combination test, to aggregate results across time points \citep{Fisher1932-je, Hedges1985-io}.
%Combination tests control the error rate assuming the null hypothesis is shared across tests.
Since we assumed that the acceptability graphs are constant, this tests the null hypothesis that the shared acceptability graph does not have a particular edge (i.e. $H_0: f \nrightarrow_{\epsilon, \cdot} f'$); The alternative hypothesis is that the edge exists.
Thus, we can replace GSTs with combination tests to achieve the desired error control.

The most severe time trend is where the acceptability graphs change frequently.
Controlling error rates in this setting is extremely difficult because previous data is not informative for future time points.
In fact, even defining an error is difficult in this regime since the relative performance of models is highly variable over time, e.g. an approval at time $t$ that looks bad at time $t + 1$ might turn out to be a very good at time $t + 2$.
As such, we recommend checking that the acceptability graphs reasonably constant before using aACP-BAC and -BABR.

\section{Efficiency of an aACP}
Just as hypothesis tests are judged by their Type I error and power, an aACP should be judged by both the rate of approving bad modifications and that for good modifications.
Using a decision-theoretic approach, we characterize the rate of good approvals by the cumulative mean of an endpoint, which we refer to as the cumulative utility.
This quantity is similar to ``regret'' in the online learning literature \citep{Shalev-Shwartz2012-vg}.
\begin{definition}
	The cumulative utility of an aACP with respect to endpoint $m$ is
	\begin{align}
	E\left[
	\frac{1}{T}
	\sum_{t=1}^T m \left ( \hat{f}_{\hat{A}_t}, \mathbb{P}_t \right)
	\right].
	\label{eq:efficiency}
	\end{align}
\end{definition}
It is not possible to design a single aACP that maximizes \eqref{eq:efficiency} for all possible model developers since we allow arbitrary unblinded adaptations.
Instead, we will evaluate the cumulative utility through a variety of simulation settings.

%\subsection{Choosing the hyper-parameters}
%
%There are some hyperparameters of our procedure.
%\begin{itemize}
%	\item Size of window
%	\item Definition of NI and size of NI margin
%	\item How much alpha to spend at each time point
%	\item How maximum number of batches to test each hypothesis
%	\item the alpha spending functions for the GSMs
%\end{itemize}
%
%What are some example values you can use?

\section{Simulations}

Through simulation studies, we evaluate the operating characteristics of the following aACPs:
{ %\small
\begin{enumerate}
	\item Blind: Approve all model updates
	\item Reset: Perform an acceptability test at level 0.05 against the last-approved model
	\item Baseline: Perform an acceptability test at level 0.05 against the initial model
	\item aACP-BAC at level $\alpha = 0.2$ with window $W = 15$
	\item aACP-BABR at level $\alpha = \alpha' = 0.2$ with window $W = 15$. The ratio of maximum wait times between the benchmark and approval was fixed at $\Delta'/\Delta = 2$.
	\item Fixed: Only approve the first model
\end{enumerate}
}
\noindent The first three aACPs have no error rate guarantees but are policies one may consider; The others provide error rate control.
In the first two simulations, we try to inflate the error rates of the aACPs.
The next two study the cumulative utility of the aACPs when proposed models are improving on average.
The last simulation explores the effects of time trends.

All simulations below consider a binary classification task where the endpoints are sensitivity and specificity.
We test for acceptability/superiority using repeated confidence intervals \citep{Cook1994-la} and Pocock alpha-spending functions \citep{Pocock1977-hw}, where $\epsilon = 0.05$ for both endpoints.
%We ran 50 replicates for each simulation.
To compare the aACPs in different scenarios, we plot endpoints of the approved models over time and show error rates in Table~\ref{table:all}.
Full simulation details are in the Appendix.

\subsection{Incremental model deterioration}
\label{sec:model_deteriorate}
In this simulation, the proposed models deteriorate gradually.
This can occur in practice for a number of reasons.
For instance, a manufacturer might try to make their SaMD simpler, cheaper, and/or more interpretable by using fewer input variables, collecting measurements through other means, or training a less complex model.
Even if their modifications are well-intended, the sponsor might end up submitting inferior models.
A model developer can also inadvertently propose adverse modifications if they repeatedly overfit to the training data.
Finally, a properly trained model can be inferior if the training data is not representative of future time points if, say, the biomarkers lose their prognostic value over time.

This simulation setup tries to induce bio-creep by submitting models that are acceptable to the currently approved model but gradually deteriorate over time.
Each proposed model is worse by $\epsilon/2$ in one endpoint and better by $\epsilon/4$ in the other.
By alternating between deteriorating the two endpoints, the manufacturer eventually submits strictly inferior models.
%We set total time $T = 200$ and the maximum wait time $\Delta = 5$ for all models.
%The number of new monitoring observations at each time point increments by ten to estimate the true performance difference with increasing precision over time, starting with 200 observations.
%An increasing number of monitoring observations also simulates increasing popularity of a SaMD.

\begin{figure}
\includegraphics[width=1.2\linewidth]{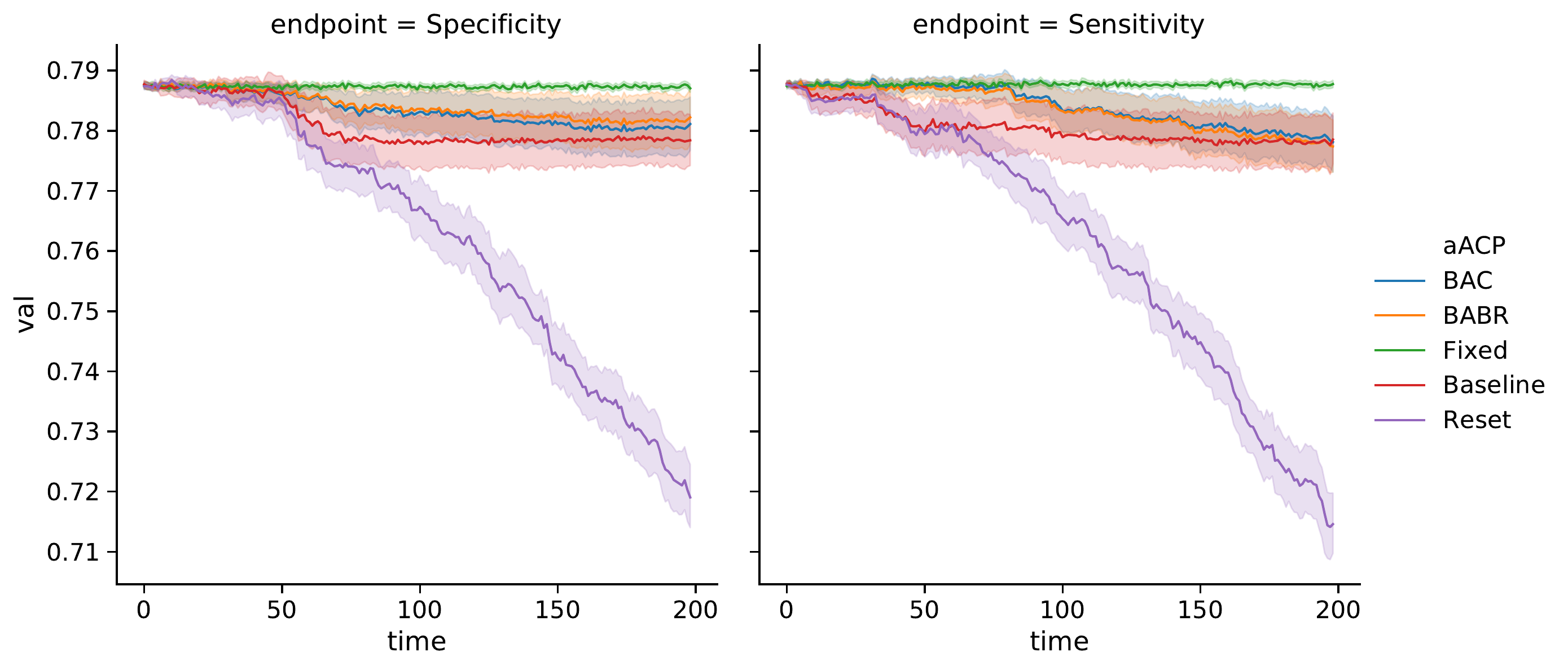}
\caption{
	Comparison of the sensitivity and specificity of models approved by different aACPs when the proposed models are gradually deteriorating.
	(We omit aACP-Blind from this plot since it would obviously perform the worst.)
}
\label{fig:adversary_biocreep}
\end{figure}

Bio-creep occurs consistently when using the aACP-Reset since it only compares against the most recently approved model (Figure~\ref{fig:adversary_biocreep}).
Both the sensitivity and specificity for the approved model at the final time point are significantly worse than the initial model.
aACP-BAC and aACP-BABR properly controlled the occurence of bio-creep since they require modifications to demonstrate acceptability with respect to \textit{all} previously approved models.

\subsection{Periodic model deterioration and improvement}
\label{sec:period}
Next we consider a simulation in which the proposed modifications periodically decline and improve in performance.
This scenario is more realistic than the previous section since a manufacturer is unlikely to only submit bad modifications.
More specifically, the proposed models monotonically improve in performance over the first fifteen time points and, thereafter, alternate between deteriorating and improving monotonically every ten time points.
%We set total time $T = 100$ and maximum wait time $\Delta = 5$ for all models.
%We accumulate 200 new observations at each time point.

As expected, aACP-Baseline had the worst error and cumulative utility.
It performed like aACP-Blind and the performances of the approved models were highly variable over time (Figure~\ref{fig:moody}).
In contrast, the other aACPs displayed much less variability and the performances were generally monotonically increasing.
aACP-Reset had the highest utility here because it performs hypothesis tests at a higher level alpha than aACP-BAC and aACP-BABR.

\begin{figure}
	\includegraphics[width=\linewidth]{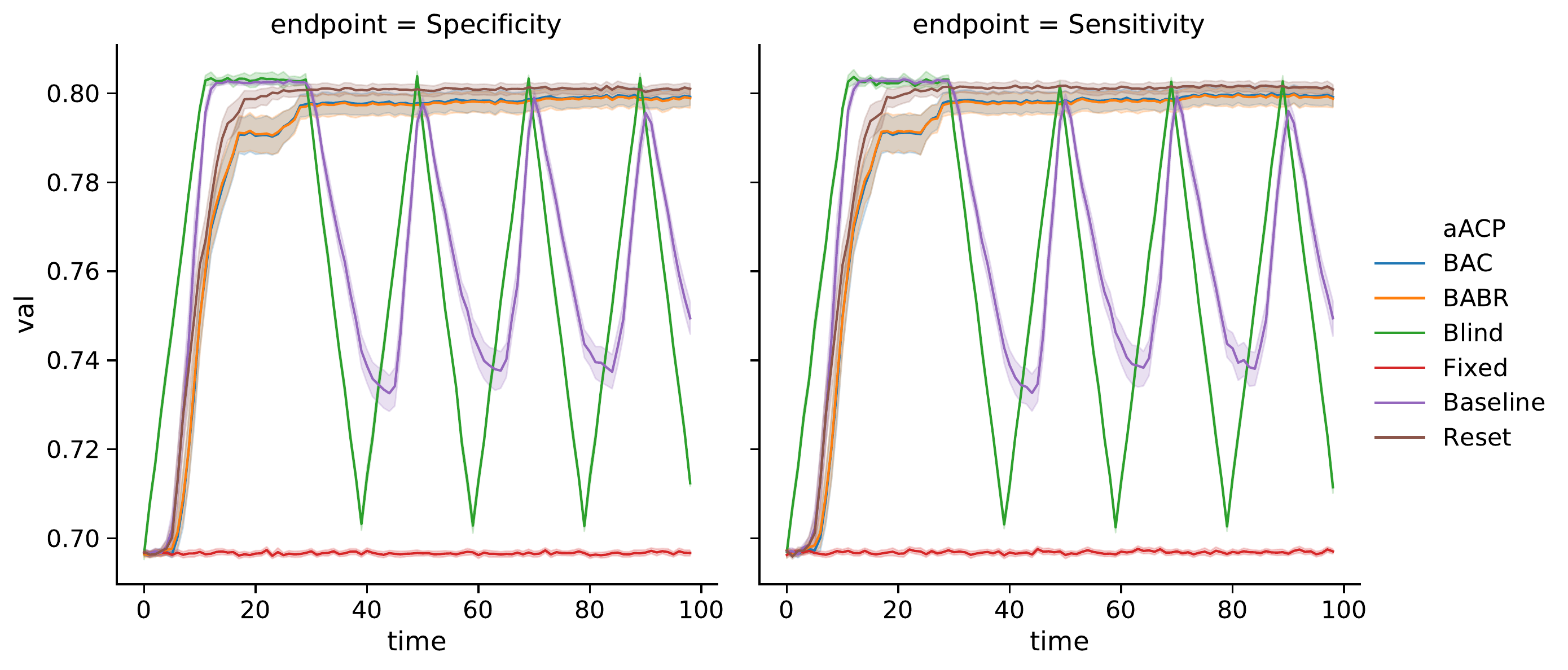}
	\caption{
	Comparison of the sensitivity and specificity of models approved by different aACPs when the proposed models periodically deteriorate and improve in performance.
%	aACP-Baseline, which only compares the candidate model to the initial model, essentially performs blind approval.
	}
	\label{fig:moody}
\end{figure}

\subsection{Accumulating data model updates}
\label{sec:accum}
We now suppose the manufacturer automatically generates modifications by training the same model on accumulating monitoring data.
%Each patient is represented by 30 covariates and the true outcome is generated using a logistic model.
In this simulation, the developer iteratively performs penalized logistic regression. % with a lasso penalty and tunes the penalty parameter using cross-validation.
Since model parameters are estimated with increasing precision, the expected improvement decreases over time and performance eventually plateaus.
As such, we investigate aACP behavior over a shorter time period.
%To increase the margin of model improvement at later time points and the ability to detect small improvements, we increase the number of training observations at each time point by five, starting with size 20, and use a larger wait time of $\Delta = 10$.
%The total time is $T = 40$ since the model performance plateaus quickly.

aACP-Blind approved good modifications the fastest; The remaining aACPs, excluding aACP-fixed, are close in cumulative utility (Figure~\ref{fig:accum}).
The similarity in performance is because less efficient aACPs can ``catch up'' in this setup: Even if an aACP has low power for detecting small improvements, the model developer will eventually propose a modification with a sufficiently large improvement that is easy to discern.
We note that aACP-BAC and -BABR behaved similarly in this simulation because models improved at a slow pace and performance plateaued over time.
As such, aACP-BABR often discovered one or no new benchmarks within a window and was unable to earn alpha-wealth much of the time.

\begin{figure}
	\includegraphics[width=\linewidth]{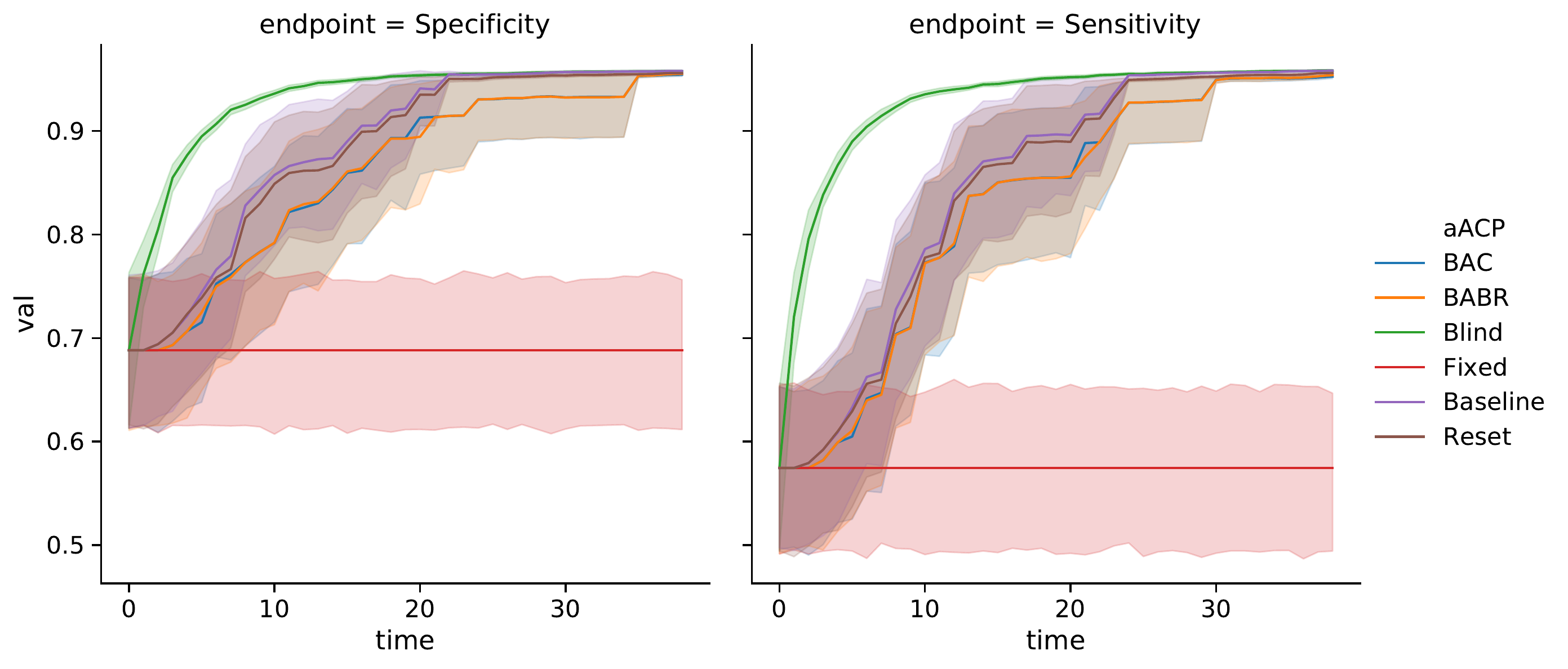}
	\caption{
	Comparison of the sensitivity and specificity of models approved by different aACPs when the model developer trains a logistic model on accumulating monitoring data.
%	Blind approval is the most efficient, though the aACPs are all quite similar.
	}
	\label{fig:accum}
\end{figure}

\subsection{Significant model improvements}
\label{sec:nice}

Next, we simulate a manufacturer that proposes models with large improvements in performance at each time point.
Large improvements usually occur when the modifications significantly change the model, such as adding a highly informative biomarker or replacing a simple linear model with a complex one that accounts for non-linearities and interactions.

Since large improvements are relatively rare, we used a short total time. % of $T = 20$. % and maximum wait time of $\Delta = 3$.
We designed the simulation to be less favorable for aACP-BAC and -BABR.
The model developer proposes a modification that improves both endpoints by 4\% compared to the most recently approved model.
Therefore an aACP cannot catch up by simply waiting for large improvements.

\begin{figure}
	\includegraphics[width=\linewidth]{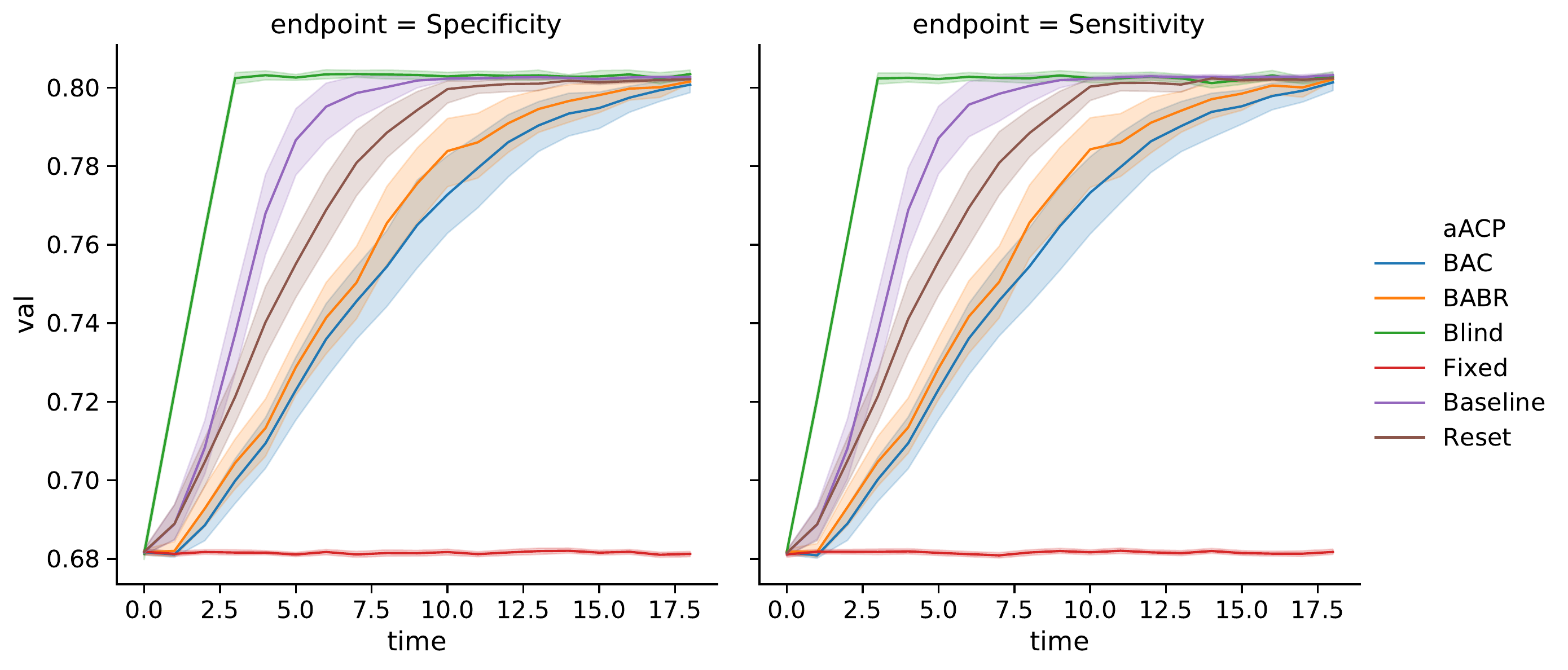}
	\caption{
	Comparison of the sensitivity and specificity of models approved by different aACPs when the model developer adaptively proposes a significantly better model than the currently approved model.
	We evaluate three different settings for aACP-BABR, where a larger index means that alpha-wealth is spent more greedily.
	}
	\label{fig:nice}
\end{figure}

As expected, Blind-aACP is the most efficient, followed by aACP-Baseline and aACP-Reset (Figure~\ref{fig:nice}).
The aACPs with error rate control are less efficient.
For example, the performance of the final models approved by aACP-BABR and aACP-Reset differed by 4\% on average.
Unlike in previous simulations, there is a clear difference between using aACP-BABR over aACP-BAP.
Since the models here improve at a fast pace, aACP-BABR earns enough alpha-wealth to discover new benchmarks with high probability.

\subsection{Robustness to time trends}
\label{sec:sim_time_trend}

Finally, we evaluate the robustness of aACP-BAC and -BABR to time trends by simulating the three time trend severity levels from Section~\ref{sec:time_trends}.
We simulate the endpoints of the proposed models to follow a sinusoidal curve.
In addition, the proposed model is always strictly inferior to the currently approved model on average.
For the \textit{graph-constant} setting, the sinusoids are aligned so that the proposed model is unacceptable at all time points.
For the \textit{graph-changing} setting, the sinusoids are offset by exactly half the period so that the proposed model is superior to the currently approved model at certain time points.
%The total time is $T = 100$ and the wait time is $\Delta = 5$.
%We use GSTs for both aACPs.

The error rates in the \textit{graph-constant} and \textit{no-trend} settings were similar (Table~\ref{table:all}), which implies that aACP-BAC and -BABR still control error rates if the acceptability graphs stay relatively constant.
However, they performed poorly in the \textit{graph-changing} setting since bad modifications appeared superior at particular time points.
%We note that aACP-BAC and -BABR behaved similarly in this simulation because the probability of discovering a benchmark was low.

\section{Discussion}

In this work, we have presented and evaluated different policies for regulating modifications to AI/ML-based SaMDs.
One of our motivations was to investigate the possibility of bio-creep, due to the parallels between this problem and noninferiority testing of new drugs.
We found that the risk of bio-creep is heightened in this regulatory problem compared to the traditional drug development setting because software modifications are easy and fast to deploy.
Nonetheless, we show that aACPs with appropriate online error-rate guarantees can sufficiently reduce the possibility of bio-creep without substantial sacrifices in our ability to approve beneficial modifications, at least in the specific settings discussed in this paper.

This paper only considers a limited scope of problems and there are still many interesting directions for future work.
One direction is to develop more efficient aACPs, perhaps by spending alpha-wealth more judiciously, discovering benchmarks using a different procedure, sequestering monitoring data for repeated testing, or considering the special case with pre-specified modifications.
Also, we have not considered aACPs that regulate modifications to SaMDs that are intended to treat and are evaluated based on patient outcomes.
%This work has also only focused on estimating the difference in endpoints between models rather than the endpoints themselves.
%Nonetheless, having an estimate of the endpoints is useful for tracking the absolute performance of the approved models.
%Since the aACPs here employ group sequential methods, one would need bias-correction methods such as in \citet{Emerson1990-mw}.

Our results raise the interesting question regarding the general structure of the regulatory policy framework.
Although aACP-BAC and -BABR mitigate the effect of bio-creep, they cannot provide indefinite error rate control without large sacrifices in cumulative utility.
So if one desires both indefinite error rate control and fast approval of good modifications, perhaps the solution is not to use a fully automated approach.
For example, human experts could perform comprehensive analyses every couple of years and the manufacturer could use an aACP in between to quickly deploy modifications.

Finally, we highlight that regulating modifications to AI/ML-based SaMDs is a highly complex problem.
This paper has primarily focused on the idealized setting of a constant diagnostic environment.
%There are still many difficult open questions---both statistical and not.
%For instance, who is in charge of executing the aACP? Should it be an independent third-party analogous to data monitoring committees used in clinical trials? Or can we run a containerized version of the software implementing the aACP?
%Also, how is monitoring data delivered to the aACP and the manufacturer in a way that minimizes the possibility of bias and misuse?
Our findings suggest that problems with bio-creep is more pervasive when modifications are designed to accommodate time trends in the patient population, available measurements, and bioclinical practice.
It is crucial that we thoroughly understand the safety risks before allowing modifications in these more complex settings.

%  The \backmatter command formats the subsequent headings so that they
%  are in the journal style.  Please keep this command in your document
%  in this position, right after the final section of the main part of
%  the paper and right before the Acknowledgements, Supplementary Materials,
%  and References sections.

\backmatter

%  This section is optional.  Here is where you will want to cite
%  grants, people who helped with the paper, etc.  But keep it short!

\section*{Acknowledgments}

The authors thank Ali Shojaie, Frederick A. Matsen IV, Holly Janes, Pang Wei Koh, and Qi Liu for helpful discussions and suggestions.
This work was supported by NIH Grant DP5OD019820.

\appendix

%  To get the journal style of heading for an appendix, mimic the following.

\section{}
\subsection{Proofs}
\begin{theorem}
	aACP-BAC achieves uniform control of $\BAP_W(\cdot)$ at level $\alpha$, i.e.
	\begin{equation}
	\BAP_W(T) \le \alpha \qquad T = 1,2,...
	\end{equation}
	\label{thrm:bac}
\end{theorem}
\begin{proof}
	At each time point, aACP-BAC launches a set of hypothesis tests comparing $\hat{f}_t$ to models with indices $\hat{M}_t = \{\hat{A}_{1}, ..., \hat{A}_{t}, \hat{A}_{t} + 1,...,t - 1\}$.
	Let the $\tilde{\mathcal{F}}_t$-measurable random variable ${G}_t$ indicate the indices of the true null hypotheses, i.e.
	$$
	\hat{G}_t = \left\{
	j \in \hat{M}_t :
	\hat{f}_{j} \nrightarrow_\epsilon \hat{f}_{t}
	\right\}.
	$$
	It is easy to see that the number of bad approvals is upper bounded by the number of incorrect rejections of the launched null hypotheses, i.e.
	\begin{align}
	\begin{split}
	&
	\sum_{1 \vee (T - W)}^T
	\mathbbm{1}\left\{
	\exists t' = 1,\cdots,t - 1
	\text{ s.t. }
	\hat{f}_{\hat{A}_{t'}} \nrightarrow_{\epsilon} \hat{f}_{\hat{A}_t}
	\right\}
	\\
	\le &
	\sum_{1 \vee (T - W)}^T
	\mathbbm{1}\left\{
	\exists j \in \hat{G}_t,
	\exists t' = 1,\cdots,\Delta_t,
	\text{ s.t. reject }
	\hat{f}_{j} \nrightarrow_{\epsilon} \hat{f}_t
	\text{ at time } t + t'
	\right\}.
	\end{split}
	\end{align}
	Taking the expectations on both sides, $\BAP_W(T)$ is upper-bounded by
	\begin{align}
	\sum_{1 \vee (T - W)}^T
	\Pr\left(
	\exists j \in \hat{G}_t,
	\exists t' = 1,\cdots,\Delta_t,
	\text{ s.t. reject }
	\hat{f}_{j} \nrightarrow_{\epsilon} \hat{f}_t
	\text{ at time } t + t'
	\right).
	\label{eq:bap_bound}
	\end{align}
	Since the hypothesis tests are tested using a gatekeeping procedure, each probability in \eqref{eq:bap_bound} is equal to the probability of rejecting the first true null hypothesis in the gatekeeping sequence.
	Thus,
	\begin{align}
	& \Pr\left(
	\exists j \in \hat{G}_t,
	\exists t' = 1,\cdots,\Delta_t,
	\text{ s.t. reject }
	\hat{f}_{j} \nrightarrow_{\epsilon} \hat{f}_t
	\text{ at time } t + t'
	\right)
	\\
	= &
	\Pr\left(
	\hat{G}_t \ne \emptyset,
	\exists t' = 1,\cdots,\Delta_t,
	\text{ s.t. reject }
	\hat{f}_{\min \hat{G}_t} \nrightarrow_{\epsilon} \hat{f}_t
	\text{ at time } t + t'
	\right)\\
	\le &
	E\left[
	\Pr\left(
	\hat{G}_t \ne \emptyset,
	\exists t' = 1,\cdots,\Delta_t,
	\text{ s.t. reject }
	\hat{f}_{\min \hat{G}_t} \nrightarrow_{\epsilon} \hat{f}_t
	\text{ at time } t + t'
	\mid \tilde{\mathcal{F}}_t
	\right)
	\right]\\
	\le & E\left[
	\hat{\alpha}_t \mathbbm{1}\left\{
	\hat{G}_t \ne \emptyset
	\right\}
	\right].
	\end{align}
	Summing together the probabilities within the window, we have
	\begin{align}
	\BAP_W(T)\le
	E\left[
	\sum_{1 \vee (T - W)}^T
	\hat{\alpha}_t
	\right]
	\le \alpha,
	\end{align}
	where the last inequality follows from the fact that $\hat{\alpha}_t$ is always selected such that
	$$
	\sum_{1 \vee (T - W)}^T
	\hat{\alpha}_t \le \alpha.
	$$
\end{proof}

\begin{theorem}
	aACP-BABR achieves uniform control of $\meBAR_{W}(\cdot)$ and $\meBSR_W(\cdot)$ at levels $\alpha$ and $\alpha'$, respectively, i.e.
	\begin{align}
	\meBAR_{W}(T) \le \alpha & \qquad \forall T = 1,2,\cdots\\
	\meBSR_{W}(T) \le \alpha' & \qquad \forall T = 1,2,\cdots
	\label{eq:mebsr}
	\end{align}
	\label{thrm:basr}
\end{theorem}
\begin{proof}
	For all $T$, $\hat{\alpha}_T$ is selected such that
	\begin{align}
	\hat{\BAR}_{W'}(T) =
	\frac{
		\sum_{t =1}^T
		\hat{\alpha}_t
		\mathbbm{1}\left\{
		t  - W \le t' + \Delta_{t'} \le t
		\right\}
	}{
		1 + \sum_{t =1 \vee (T - W')}^{T-1}
		\mathbbm{1} \left\{
		\hat{B}_{t - 1} \ne \hat{B}_t
		\right\}
	}
	\le \alpha
	\qquad \forall W' = 1,...,W.
	\label{eq:fdr_constraint}
	\end{align}
	Note that we can always set $\hat{\alpha}_T = 0$ to satisfy these constraints, assuming that \eqref{eq:bar_control} was satisfied at times $t = 1,...,T - 1$.
	Using the result in the proof of Theorem~\ref{thrm:bac}, we then bound the numerator of $\BAR_{W}(T)$ as follows
	\begin{align}
	& E\left[
	\sum_{t =1 \vee (T - W)}^T
	\mathbbm{1}\left\{
	\exists t' = 1,...,t - 1 \text{ s.t. }
	\hat{f}_{\hat{A}_{t'}}
	\nrightarrow_{\epsilon, t}
	\hat{f}_{\hat{A}_{t}}
	\right\}
	\right] \\
	& \le E\left[
	\sum_{t =1}^T
	\hat{\alpha}_t
	\mathbbm{1}\left\{
	t  - W \le t' + \Delta_{t'} \le t
	\right\}
	\right ]\\
	& \le E\left[
	\alpha\left(
	1 + \sum_{t =1 \vee (T - W)}^T
	\mathbbm{1} \left\{
	\hat{B}_{t - 1} \ne \hat{B}_t
	\right\}
	\right)
	\right],
	\end{align}
	where the last line follows from \eqref{eq:fdr_constraint}.
	Rearranging, we get that $\meBAR_{W}(T) \le \alpha$.
	The proof for uniform control of $\meBSR_{W}(\cdot)$ is essentially the same, where we replace the alpha-spending function with $\alpha'_t$ and the threshold with $\alpha'$.
\end{proof}

\subsection{Simulation settings}
We ran 50 replicates for each simulation.
\subsubsection{Incremental deterioration}
We set total time $T = 200$ and the maximum wait time $\Delta = 5$ for all models.
The number of new monitoring observations at each time point increments by ten to estimate the true performance difference with increasing precision over time, starting with 200 observations.

\subsubsection{Periodic model deterioration and improvement}
We set total time $T = 100$ and maximum wait time $\Delta = 5$ for all models.
We accumulate 200 new observations at each time point.

\subsubsection{Accumulating data}
Each patient is represented by 30 covariates and the true outcome is generated using a logistic model.
The developer performs logistic regression with a lasso penalty and tunes the penalty parameter using cross-validation.
To increase the margin of model improvement at later time points and the ability to detect small improvements, we increase the number of training observations at each time point by five, starting with size 20, and use a larger wait time of $\Delta = 10$.
The total time is $T = 40$ since the model performance plateaus quickly.

\subsubsection{Significant model improvements}
In order to make the model improvements significant with high probability, we accumulate 650 observations at each time point, which is more than the other simulation settings.
Since large improvements are relatively rare, we used a short total time of $T = 20$.
Since a company is likely more confident in these improvements, the maximum wait time is set to $\Delta = 3$.

\subsubsection{Time trends}
The total time is $T = 100$ and the wait time is $\Delta = 5$.
We accumulate 300 new observations at each time point.

\begin{algorithm}
	\SetAlgoLined
	\For{$t = 1 , 2, ...$}{
		$\hat{A}_{t} = \hat{A}_{t - 1}$\;
		\tcc{Determine if there are new approvals}
		\For{$j = \hat{A}_{t - 1} + 1, ..., t - 1 $}{
			\If{$t \le j + \Delta_j$ \tcp*{If $\hat{f}_{j}$ is under consideration for approval}}{
				Run $\epsilon$-acceptability tests: Test null hypotheses $\hat{f}_{j'} \nrightarrow_{\epsilon} \hat{f}_{j}$ for $j' = \hat{A}_{1}, ..., \hat{A}_{j - 1}, \hat{A}_{j - 1} + 1,...,\hat{A}_{t-1}$ with critical value $c_j(t)$ in gatekeeping style\;
				\If{All $\epsilon$-acceptability tests pass}{
					$\hat{A}_{t} = j$\;
				}
			}
		}
		\tcc{Launch new hypothesis tests for new model proposal}
		Launch family of $\epsilon$-acceptability tests with null hypotheses $\hat{f}_{j} \nrightarrow_{\epsilon} \hat{f}_{t}$ for $j = \hat{A}_{1}, ... \hat{A}_t, \hat{A}_t + 1, ..., t - 1$\;
		Choose $\hat{\alpha}_t$ such that \eqref{eq:bap_control} is satisfied\;
		Select alpha-spending function for $\hat{\alpha}_t$ and its critical value function $c_t(\cdot)$ over the next $\Delta_t$ time points.\;
		
	}
	\caption{aACP-BAC}
	\label{algo:aacp_count}
\end{algorithm}

\begin{algorithm}
	\SetAlgoLined
	\For{$t = 1 , 2, ...$}{
		$\hat{A}_{t} = \hat{A}_{t - 1}$\;
		\tcc{Determine if there are new approvals}
		\For{$j = \hat{A}_{t - 1} + 1, ..., t - 1 $}{
			\If{$t \le j + \Delta_j$ \tcp*{If $\hat{f}_{j}$ is under consideration for approval}}{
				Run $\epsilon$-acceptability tests: Test null hypotheses $\hat{f}_{j'} \nrightarrow_{\epsilon} \hat{f}_{j}$ for $j' = \hat{A}_{j}, ..., \hat{A}_{t - 1}$ with critical value $c_j(t)$ in gatekeeping style\;
				\If{All $\epsilon$-acceptability tests pass}{
					$\hat{A}_{t} = j$\;
				}
			}
		}
		$\hat{B}_{t} = \hat{B}_{t - 1}$\;
		\tcc{Determine if there are new benchmarks}
		\For{$j = \hat{A}_{1}, ..., \hat{A}_{t-1} $}{
			\If{$j > \hat{B}_{t-1}$ and $t \le j + \Delta_j$ \tcp*{If $\hat{f}_{j}$ is under consideration for approval}}{
				Run superiority tests: Test null hypotheses $\hat{f}_{j'} \nrightarrow_{0} \hat{f}_{j}$ for $j' = \hat{B}_{j}, ..., \hat{B}_{t - 1}$ with critical value $c'_j(t)$ in gatekeeping style\;
				\If{All superiority tests pass}{
					$\hat{A}_{t} = j$\;
				}
			}
		}
		\tcc{Launch new hypothesis tests for new model proposal}
		Launch family of $\epsilon$-acceptability tests with null hypotheses $\hat{f}_{j} \nrightarrow_{\epsilon} \hat{f}_{t}$ for $j = \hat{A}_{1}, ... \hat{A}_t, \hat{A}_t + 1, ..., t - 1$\;
		Launch family of superiority tests with null hypotheses $\hat{f}_{j} \nrightarrow_{\epsilon} \hat{f}_{t}$ for $j = \hat{B}_{t}, ..., t - 1$\;
		Choose $\hat{\alpha}_t, \hat{\alpha}_t'$ such that \eqref{eq:bar_control} and \eqref{eq:bsr_control} are satisfied\;
		Select alpha-spending function for $\hat{\alpha}_t$ and $\hat{\alpha}_t'$ and their critical value functions $c_t(\cdot), c_t'(\cdot)$ over the next $\Delta_t$ time points.\;
		
	}
	\caption{aACP-BABR}
	\label{algo:aacp_ratio}
\end{algorithm}

\pagebreak
\begin{landscape}
\begin{table}
	\begin{tabular}{l|rHHrHH|rr|rr|rr}
		\toprule
		{} &    &   &  & &  & &  \multicolumn{2}{c}{Final} &  \multicolumn{2}{c}{Cumulative utility} &  &  \\
		approver &  $BAC_W$ &  num\_bad\_approv\_tot &  num\_bad\_std &  \# approved &  num\_approved\_window &  num\_std\_window &  Specificity &  Sensitivity &  Specificity &  Sensitivity &  $\meBAR_{W}$ &  $\meBSR_{W}$ \\
		\midrule
		\multicolumn{13}{c}{\textit{Incremental model deterioration, Section~\ref{sec:model_deteriorate}}}\\
		BABR     &             0.000 &               0.000 &        0.000 &             1.940 &                1.020 &           1.000 &              0.782 &              0.777 &                   0.784 &                   0.784 &  0.000 &  0.000 \\
		BAC      &             0.000 &               0.000 &        0.000 &             1.960 &                1.000 &           1.000 &              0.781 &              0.778 &                   0.783 &                   0.784 &  0.000 &  0.000 \\
		Baseline &             0.060 &               0.120 &        0.000 &             2.220 &                1.160 &           1.000 &              0.778 &              0.779 &                   0.781 &                   0.781 &  0.060 &  0.000 \\
		Fixed    &             0.000 &               0.000 &        0.000 &             1.000 &                1.000 &           1.000 &              0.787 &              0.788 &                   0.787 &                   0.788 &  0.000 &  0.000 \\
		Reset    &             1.180 &               7.780 &        1.180 &             9.860 &                1.180 &           1.180 &              0.719 &              0.715 &                   0.763 &                   0.761 &  0.541 &  0.541 \\
		\midrule
		\multicolumn{13}{c}{\textit{Periodic model deterioration/improvement, Section~\ref{sec:period}}}\\
		BABR     &             0.000 &               0.000 &        0.000 &             2.920 &                2.360 &           1.800 &              0.799 &              0.799 &                   0.786 &                   0.787 &  0.000 &  0.000 \\
		BAC      &             0.000 &               0.000 &        0.000 &             2.940 &                2.320 &           1.000 &              0.799 &              0.799 &                   0.786 &                   0.787 &  0.000 &  0.000 \\
		Baseline &             6.300 &              24.640 &        0.000 &            43.360 &               10.080 &           1.000 &              0.749 &              0.749 &                   0.765 &                   0.766 &  6.300 &  0.000 \\
		Blind    &            15.000 &              66.000 &        0.000 &            99.000 &               15.000 &           1.000 &              0.712 &              0.711 &                   0.762 &                   0.762 & 15.000 &  0.000 \\
		Fixed    &             0.000 &               0.000 &        0.000 &             1.000 &                1.000 &           1.000 &              0.697 &              0.697 &                   0.697 &                   0.697 &  0.000 &  0.000 \\
		Reset    &             0.020 &               0.020 &        0.020 &             3.740 &                2.860 &           2.860 &              0.801 &              0.801 &                   0.790 &                   0.791 &  0.020 &  0.020 \\
		\midrule
		\multicolumn{13}{c}{\textit{Training on accumulating data, Section~\ref{sec:accum}}}\\
		BABR     &             0.040 &               0.040 &        0.000 &             3.460 &                2.120 &           1.560 &              0.955 &              0.954 &                   0.858 &                   0.825 &  0.025 &  0.000 \\
		BAC      &             0.020 &               0.020 &        0.000 &             3.240 &                2.080 &           1.000 &              0.954 &              0.952 &                   0.858 &                   0.825 &  0.020 &  0.000 \\
		Baseline &             0.100 &               0.120 &        0.000 &            22.120 &               11.080 &           1.000 &              0.958 &              0.958 &                   0.884 &                   0.845 &  0.100 &  0.000 \\
		Blind    &             3.840 &               3.960 &        0.000 &            39.000 &               15.000 &           1.000 &              0.958 &              0.958 &                   0.928 &                   0.922 &  3.840 &  0.000 \\
		Fixed    &             0.000 &               0.000 &        0.000 &             1.000 &                1.000 &           1.000 &              0.688 &              0.574 &                   0.688 &                   0.574 &  0.000 &  0.000 \\
		Reset    &             0.080 &               0.080 &        0.380 &             4.520 &                2.620 &           2.620 &              0.956 &              0.956 &                   0.879 &                   0.840 &  0.028 &  0.168 \\
		\midrule
		\multicolumn{13}{c}{\textit{Significant model improvement, Section~\ref{sec:nice}}}\\
		BABR     &             0.000 &               0.000 &        0.000 &             4.020 &                3.900 &           2.720 &              0.802 &              0.802 &                   0.757 &                   0.757 &  0.000 &  0.000 \\
		BAC      &             0.000 &               0.000 &        0.000 &             3.980 &                3.800 &           1.000 &              0.801 &              0.801 &                   0.753 &                   0.753 &  0.000 &  0.000 \\
		Baseline &             0.000 &               0.000 &        0.000 &            17.160 &               14.760 &           1.000 &              0.803 &              0.803 &                   0.778 &                   0.779 &  0.000 &  0.000 \\
		Blind    &             0.000 &               0.000 &        0.000 &            19.000 &               15.000 &           1.000 &              0.803 &              0.803 &                   0.790 &                   0.790 &  0.000 &  0.000 \\
		Fixed    &             0.000 &               0.000 &        0.000 &             1.000 &                1.000 &           1.000 &              0.681 &              0.682 &                   0.682 &                   0.682 &  0.000 &  0.000 \\
		Reset    &             0.000 &               0.000 &        0.020 &             4.420 &                4.200 &           4.200 &              0.802 &              0.802 &                   0.770 &                   0.771 &  0.000 &  0.007 \\
		\midrule
		\multicolumn{13}{c}{\textit{Time trend experiment, Section~\ref{sec:sim_time_trend}}}\\
		no-trend-BAC         &             0.020 &               0.000 &        0.000 &             1.020 &                1.020 &           1.000 &              0.712 &              0.712 &                   0.712 &                   0.712 &  0.000 &  0.000 \\
		no-trend-Fixed       &             0.000 &               0.000 &        0.000 &             1.000 &                1.000 &           1.000 &              0.712 &              0.712 &                   0.712 &                   0.712 &  0.000 &  0.000 \\
		graph-constant-BAC   &             0.020 &               0.020 &        0.000 &             1.020 &                1.000 &           1.000 &              0.651 &              0.651 &                   0.712 &                   0.712 &  0.020 &  0.000 \\
		graph-constant-Fixed &             0.000 &               0.000 &        0.000 &             1.000 &                1.000 &           1.000 &              0.651 &              0.652 &                   0.712 &                   0.712 &  0.000 &  0.000 \\
		graph-changing-BAC   &             --- &               0.060 &        0.000 &             1.820 &                1.140 &           1.000 &              0.632 &              0.633 &                   0.687 &                   0.687 &  --- &  --- \\
		graph-changing-Fixed &             --- &               0.000 &        0.000 &             1.000 &                1.000 &           1.000 &              0.772 &              0.773 &                   0.711 &                   0.712 &  --- &  --- \\
		\bottomrule
	\end{tabular}
	\caption{
		Comparison of aACP in different simulation settings.
		Columns $\BAP_W, \meBAR_{W}, \meBSR_{W}$ display the maximum error rate over all time points.
		In the incremental model deterioration, we omit aACP-Blind since it always converges to completely uninformative classifier.
		In the time trend experiments, we omit results from aACP-BABR because they are very similar to aACP-BAC.
	}
	\label{table:all}
\end{table}
\end{landscape}

\bibliographystyle{biom}
\bibliography{main}

\end{document}